\documentclass{article}
\usepackage{lmodern}
\usepackage{arxiv}
\usepackage[utf8]{inputenc}
\usepackage[T1]{fontenc}
\usepackage{microtype}

\usepackage{amsmath,amsfonts,bm}

\def\eqref#1{equation~\ref{#1}}

\def\1{\bm{1}}

\DeclareMathAlphabet{\mathsfit}{\encodingdefault}{\sfdefault}{m}{sl}
\SetMathAlphabet{\mathsfit}{bold}{\encodingdefault}{\sfdefault}{bx}{n}

\def\gP{{\mathcal{P}}}
\def\gQ{{\mathcal{Q}}}

\usepackage{natbib}
\setcitestyle{authoryear,round,citesep={;},aysep={,},yysep={;}}
\usepackage{url}
\usepackage{array}
\usepackage{tabularx}
\usepackage{graphicx}
\usepackage{placeins}
\usepackage{tikz}
\usetikzlibrary{arrows.meta}
\usepackage{soul}
\usepackage{algorithm}
\usepackage{algorithmic}
\usepackage{listings}
\usepackage[hidelinks,pdfusetitle]{hyperref}

\newtheorem{proposition}{Proposition}
\newtheorem{lemma}{Lemma}

\title{LCT: Losslessly Compressed Tensors with GPU friendly execution for Training}
\renewcommand{\undertitle}{}
\renewcommand{\headeright}{}
\renewcommand{\shorttitle}{LCT: Losslessly Compressed Tensors}
\author{%
  Max Zhu\textsuperscript{1}\thanks{Corresponding author: Max Zhu (\href{mailto:mz406@cam.ac.uk}{\texttt{mz406@cam.ac.uk}}).}\qquad
  Bohang Sun\textsuperscript{1}\qquad
  Pietro Liò\textsuperscript{1}\qquad
  Jialin Yu\textsuperscript{2}\\[0.6em]
  {\normalfont\normalsize
    \textsuperscript{1}\,University of Cambridge\qquad
    \textsuperscript{2}\,University of Oxford}%
}
\hypersetup{pdfauthor={Max Zhu, Bohang Sun, Pietro Liò, Jialin Yu}}
\date{}

\begin{document}
\begingroup
  \renewcommand{\thefootnote}{}
  \footnotetext[0]{Preprint, under review.\\
    \hspace*{1.8em}Code and package: \url{https://github.com/Maxzhu123/Losslessly_Compressed_Tensors}.}
\endgroup
\setcounter{footnote}{1}
\maketitle
\setcounter{footnote}{0}
\begin{abstract}

Lossless tensor compression could reduce memory usage during training, but training requires tensors to be repeatedly re-encoded.
We introduce \textbf{L}osslessly \textbf{C}ompressed \textbf{T}ensors (LCT), a GPU execution framework designed for this repeated encode–decode regime. LCT exploits the natural redundancy in floating point data to reduce storage for weights, activations and optimizer states. Highly optimized GPU kernels implement Huffman coding with precomputed codebooks for low compression and decompression overhead, essential for training. Our analysis characterizes LCT's storage costs and adaptation to magnitude shifts. LCT achieves a compression ratio of up to $1.455\times$, and is demonstrated in training and finetuning settings. On a 162M-parameter NanoGPT, LCT reduces memory usage by $24.6\%$ for a $7.2\%$ step-time increase. In Qwen3-4B fine-tuning, LCT saves up to $23.8\%$ of peak memory, enabling a 3.5X larger context length in training on the same GPU.
\end{abstract}

\suppressfloats[t] 
\begin{figure}[t]
 \centering
 \includegraphics[width=\linewidth]{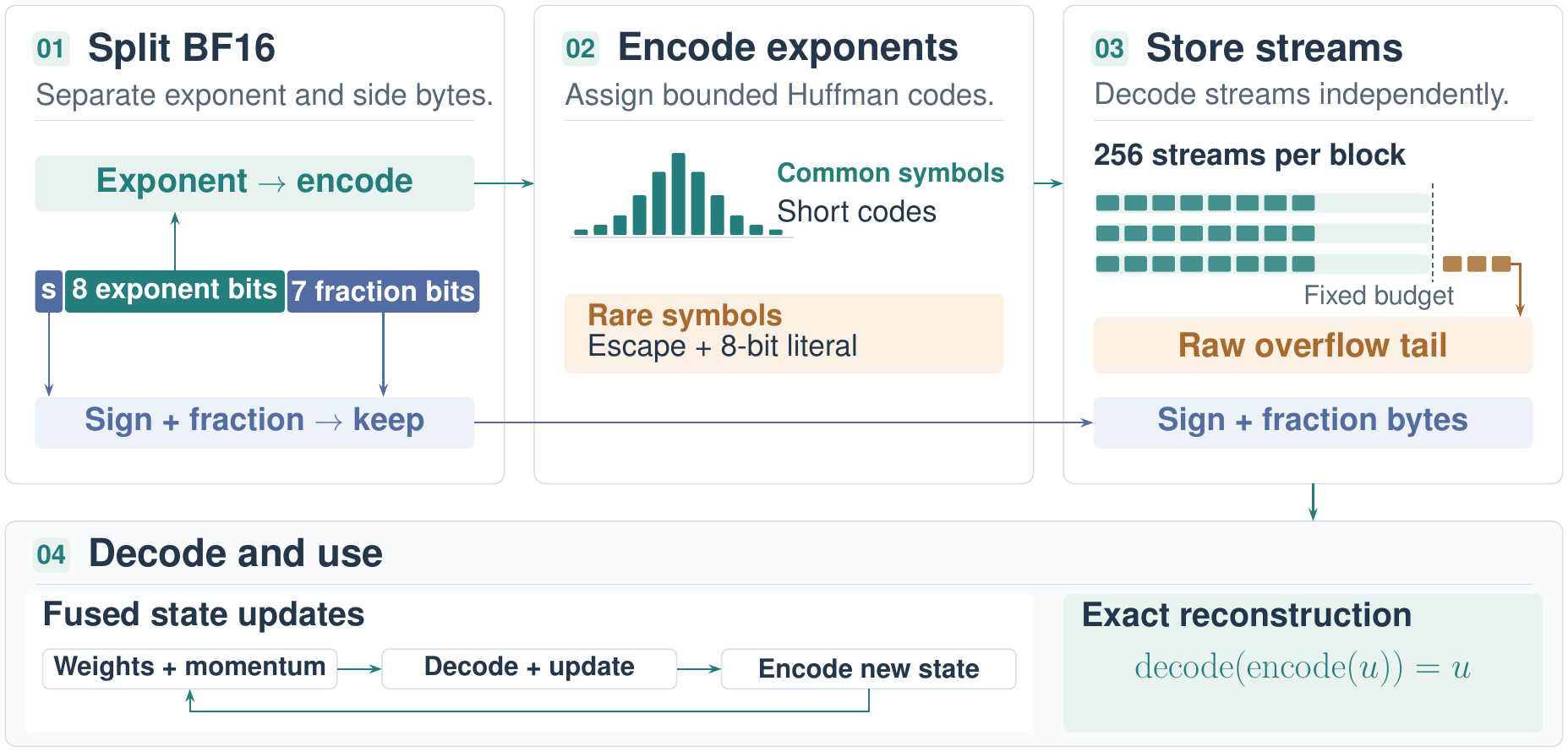}
 \caption{Overview of LCT. Sign-mantissa bytes remain unchanged while exponents are encoded with bounded Huffman codes with escapes for rare values. Values are packed into fixed sized streams with overflow buffers. Tensors are decoded losslessly, with fused kernels for optimizer updates.}
 \label{fig:graphical-abstract}
\end{figure}
\section{Introduction}
\label{sec:introduction}

Memory is a limiting factor in training large language models (LLMs), larger models and longer sequence lengths have increased memory requirements. Weights, optimizer buffers and transient activation caches all consume memory. Although a slower device can be compensated for with longer training, a memory limitation is much more difficult to overcome. Offloading storage to the CPU \citep{cpu_offload, cpu_offload2}, or even NVMe \citep{nvme_offload} allows for larger models, but can easily suffer from IO bottlenecks.  

The dominant methods to tackle this trade off accuracy for storage. Low precision training datatypes such as FP8 \citep{fp8} and FP4 \citep{fp4} use fewer bits per value. However, intermediate activations and caches still use higher precision to ensure numerical stability and hardware support is limited to the most modern hardware. Activation quantization has been explored down to 2 bits per value in ActNN \citep{actnn}, while other methods exploited the low rank nature of activations \citep{compact}. Optimizer states have been quantized to 8 bits \citep{8bit_optim}. 3-4 bit weight quantization methods such as GPTQ \citep{gptq} and AWQ \citep{awq} are widely adopted for inference. These methods are designed to minimise accuracy loss, but are not entirely lossless and require careful tuning to avoid accuracy degradation.

Lossless compression is an alternative. Gist \citep{gist} applies lossless compression to ReLU activations, but is not universally applicable. General methods have been developed to take advantage of compression for inference \citep{dfloat11, zipnn, splitzip, zipserv}. However, these methods have not been applied for training. Unlike inference, weights and optimizer states are updated every step, and activations must persist between the forward and backward pass. Table~\ref{tab:training-lifecycle} summarizes these lifetimes and the corresponding operations for an SGD-momentum (SGDM) step. Existing methods for lossless compression offer fast decoding speeds, but are too slow to encode (Table~\ref{tab:codec-bandwidth}). Training turns lossless compression into a repeated encode–decode problem rather than a decode-centric one. Figure~\ref{fig:compression-use-cases} illustrates this distinction.

We introduce \textbf{L}osslessly \textbf{C}ompressed \textbf{T}ensors (LCT) (Figure~\ref{fig:graphical-abstract}). Designed for high performance on GPU hardware, LCT losslessly compresses the exponents of Bfloat16 (BF16) values to reduce memory usage during training. Fused pointwise operators are introduced for optimizers with compressed buffers, while a shared GPU buffer removes host synchronization during training. PyTorch autograd integration compresses cached activations. The tensors remain compressed until they are needed, at which point they are decompressed on-demand. 

Our contributions are:
\begin{itemize}
 \item A lossless BF16 codec and training integration, including GPU Triton kernels, fused pointwise updates, and a preallocated buffer.
 \item An analysis of coding and allocated storage costs under model mismatch, together with bounds explaining how exponent centering compensates for magnitude changes during training.
 \item An evaluation spanning codec throughput, fused operators, and training memory. LCT encodes $2.38\times$ faster than the next-fastest evaluated codec on an RTX 3090 and reduces peak memory by over $20\%$ in training and fine tuning applications.
\end{itemize}

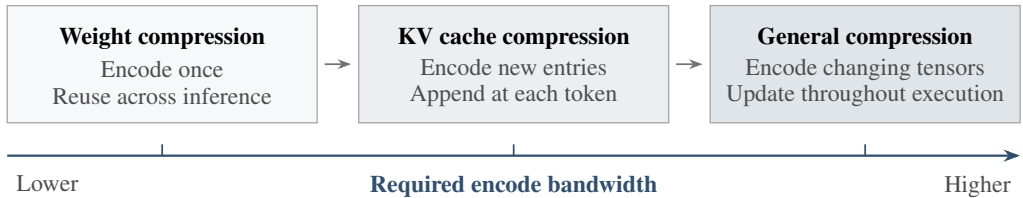
\begin{figure}[htbp]
 \centering
\begin{tikzpicture}[
    x=1cm,y=1cm,
    font=\small,
    panel/.style={draw=black!35, line width=0.5pt,
        minimum width=4.10cm, minimum height=1.55cm,
        inner sep=0pt},
    heading/.style={font=\small\bfseries,align=center},
    detail/.style={font=\footnotesize,text=black!70,align=center},
    flow/.style={-{Stealth[length=1.8mm,width=1.3mm]},
        draw=black!55,line width=0.65pt}
]
\definecolor{lctblue}{RGB}{48,79,111}

\node[panel,fill=lctblue!4] at (2.05,1.55) {};
\node[panel,fill=lctblue!9] at (6.70,1.55) {};
\node[panel,fill=lctblue!15] at (11.35,1.55) {};

\node[heading] at (2.05,1.90) {Weight compression};
\node[detail] at (2.05,1.31) {Encode once\\Reuse across inference};
\node[heading] at (6.70,1.90) {KV cache compression};
\node[detail] at (6.70,1.31) {Encode new entries\\Append at each token};
\node[heading] at (11.35,1.90) {General compression};
\node[detail] at (11.35,1.31) {Encode changing tensors\\Update throughout execution};

\draw[flow] (4.20,1.55) -- (4.55,1.55);
\draw[flow] (8.85,1.55) -- (9.20,1.55);

\draw[-{Stealth[length=2.2mm,width=1.6mm]},
    draw=lctblue,line width=0.85pt] (0,0.34) -- (13.40,0.34);
\foreach \x in {2.05,6.70,11.35}
    \draw[draw=lctblue,line width=0.65pt] (\x,0.34) -- (\x,0.49);
\node[anchor=north west,font=\footnotesize,text=black!70] at (0,0.20) {Lower};
\node[anchor=north east,font=\footnotesize,text=black!70] at (13.40,0.20) {Higher};
\node[anchor=north,font=\small\bfseries,text=lctblue]
    at (6.70,0.20) {Required encode bandwidth};
\end{tikzpicture}
 \caption{Compression use cases ordered by increasing encode-bandwidth requirement. Static inference weights amortize encoding across repeated use; autoregressive KV caches add entries at each token; general tensor compression must accommodate frequent updates. The ordering is qualitative and reflects these access patterns, rather than measured throughput.}
 \label{fig:compression-use-cases}
\end{figure}

\section{Related Work}
\paragraph{GPU-accelerated lossless compression.} Several methods have investigated using lossless compression on the GPU. The nvCOMP library \citep{nvcomp} is a GPU-optimized library that implements several widely used compression methods, including LZ4 \citep{lz4}, as well as proprietary methods for floating-point data such as Bitcomp and Cascaded. Fast GPU-optimized ANS coding \citep{ans_coding} was explored by \cite{ans_gpu}. Huffman encoding has emerged as an efficient GPU compression and decompression method. \cite{huffman_gap} explored using gap arrays for efficient compression and decoding. \cite{huffman_gpu} explored optimizing the full Huffman encoding and decoding process on GPUs.

\paragraph{Lossless compression for inference.} Lossless weight compression techniques have been explored, such as DFloat11 \citep{dfloat11}, which utilizes Huffman encoding with gap arrays for fast GPU decoding. ZipNN \citep{zipnn} uses similar principles. Both methods encode the exponent, with the encoding process done on the CPU. This allows for fast decoding, but encoding is comparatively slow. ZipServ \citep{zipserv} combines a tensor-core-aware bitmap format with ZipGEMM, which fuses decompression with matrix multiplication. \cite{huffman_ml} and \cite{zipccl} use GPU-native Huffman codes for network communication. Recently, SplitZip \citep{splitzip} proposed lossless KV-cache compression by replacing exponents with 4-bit symbols and using escape values for rarer cases. While simple, its encoding achieves suboptimal compression. Appendix~\ref{app:split_zip} provides a comparison between LCT and SplitZip.

\section{Compressibility in tensors}
\label{sec:background_motivation}

\begin{table}[t]
\centering
\setlength{\tabcolsep}{3pt}
\renewcommand{\arraystretch}{1.15}
\caption{Tensor lifecycle and LCT codec operations per SGD-with-momentum step, after momentum initialization. The weight row covers MLP and other layer weights, assuming one dense use in forward and one in backward. Encodes include output recompression. Dense decodes materialize a tensor; fused decodes consume it within an update. Counts describe logical operations, not kernel launches.}
\label{tab:training-lifecycle}
\small
\begin{tabularx}{\linewidth}{@{}lccc>{\raggedright\arraybackslash}X@{}}
\hline
Tensor & Encodes & \shortstack{Dense\\decodes} & \shortstack{Fused\\decodes} & Uses and lifetime \\
\hline
Weights $W$ & 1 & 2 & 1 & Forward and backward computation; fused $W\leftarrow W-\eta M$. Persist across steps. \\
Biases & 0 & 0 & 0 & Kept dense because they are small; updated with dense operations. \\
Momentum $M$ & 1 & 0 & 2 & Fused $M\leftarrow\beta M+G$, then fused $W\leftarrow W-\eta M$. Persists across steps. \\
Activations & 1 & 1 & 0 & Activations are encoded in forward and decoded for gradient computation in backward. \\
\hline
\end{tabularx}
\end{table}

\subsection{Exponent compressibility in floating point data}
\label{sec:full-value-coding}
The structure of floating-point tensors provides an opportunity for lossless compression. The BF16 data type consists of 1 sign bit, 8 exponent bits, and 7 mantissa bits \citep{bf16}. Typical data is well described as the quantization of a concentrated continuous distribution, the sign and mantissa have high entropy, while the exponent has lower entropy. Compressing the exponent only is sufficient, for example, a standard Gaussian rounded to BF16 has exponent entropy $2.545$ bits/symbol, while the full distribution's entropy is $10.464$ bits/symbol. Encoding the full BF16 value saves only 0.081 bits/value compared with compressing only the exponent. Appendix~\ref{app:full-value-entropy} gives the calculation. Appendix~\ref{app:exponent-distributions} explains why this is the case.

\subsection{Structure in weights and activations}
\label{sec:weight-distribution-motivation}
The distributions encountered in neural networks are compressible.
Figure~\ref{fig:weight-exponents} shows the exponent distributions of the six largest parameter groups in the Nemotron-H-8B LLM \citep{nemotron_h}. All distributions are highly concentrated around their peaks. In each group, just 9--11 of the 256 exponent symbols account for at least 99\% of the weights, with marginal entropy of 2.58--2.70 bits. The compressibility of the weights has been noted in several works \citep{zipnn, dfloat11, zipserv}.

A similar trend emerges in the activations. Figure~\ref{fig:activation-exponents} shows the exponent distributions of five Nemotron-H-8B activation groups, aggregated across layers and four 1024-token sequences. Again, the distributions are highly concentrated, though there is more variation. In particular, the squared ReLU activation has high sparsity, with 93.23\% being exactly zero. Exponent entropy ranges from 0.63 bits for the zero-rich FFN hidden state to 3.91 bits for Mamba hidden activations (Appendix~\ref{app:activation-exponents}).

\section{Methodology}
\label{sec:method}
\label{sec:features}

\subsection{Bounded exponent codes}
\label{sec:representation}

\paragraph{Coding background.}
Huffman coding \citep{huffman1952} assigns variable-length prefix-free codewords, with shorter codes for more probable symbols. For a matched source model, its expected length is less than one bit per symbol above entropy, but individual codes can be as long as the alphabet size minus one. Length-limited Huffman coding \citep{bounded_huffman} bounds this worst case, and Package-Merge efficiently constructs the optimal code under the length constraint. Escape coding \citep{escape_huffman} groups rare symbols under a shared prefix followed by a literal, freeing shorter codes for common symbols. LCT uses bounded Huffman codes with escapes to combine compact storage with simple parallel decoding

Arithmetic coding \citep{arithmetic_coding} and ANS \citep{ans_coding} can approach the entropy limit without the integer-length constraint on individual symbols, but are slower to implement on the GPU, as we show in the comparisons. 

\paragraph{LCT code construction.}
LCT splits each BF16 value into an exponent byte and an unchanged sign--fraction byte. We write $\gP=(p_t)$ for the centered distribution of exponent symbols in the tensor, and $\gQ=(q_t)$ for the fixed coding model.

The user selects $\gQ$ based on prior knowledge of the tensor distribution. It is not estimated at runtime, avoiding codebook construction during training. At runtime, a sampled exponent center shifts the codebook to adapt to the tensor scale (Section~\ref{sec:standard-codec}). The codebooks are precomputed before running and a mismatch between $\gP$ and $\gQ$ can increase storage.

LCT adopts bounded Huffman coding with escape coding. Prefixes are capped at eight bits and rare escaped symbols are encoded with an escape prefix followed by an eight-bit literal. Without escapes, coding all 256 symbols under this cap would force every code to have length eight. Escapes free short prefixes for common exponents.

The codebook construction is based on Package-Merge \citep{bounded_huffman}. LCT further adapts this algorithm to handle escape values. We sort symbols by probability $q_t$ and test every tail cutoff. Package-Merge optimizes prefix lengths for each candidate. The final code minimizes
\begin{equation}
 \mathcal{C}(D)=\sum_{t\in D}q_t\ell_\gQ(t)
 +q_{\mathrm{esc}}\left(\ell_{\mathrm{esc}}+8\right),
 \label{eq:escape-search}
\end{equation}
where $D$ contains directly coded symbols, $\ell_\gQ(t)$ is the length $\gQ$ assigns to symbol $t$, $\ell_{\mathrm{esc}}$ is the escape-prefix length, and $q_{\mathrm{esc}}$ is the escaped tail mass. Algorithm~\ref{alg:codebook} gives the full algorithm, with an optimality proof given in Appendix~\ref{app:tables}.

\subsection{Storage structure}
\label{sec:streams}
Storage consists of two components. The sign and fraction bits are stored together as uncompressed sign--fraction bytes. The encoded exponents are stored in fixed-size \emph{blocks}. 

Each block holds $B$ exponent symbols. The clean and sparse configurations use $B=65{,}536$ and the medium and high configurations use $32{,}768$ (Table~\ref{tab:geometry}). Blocks tile the tensor, with padding. Each block is divided into $L=256$ \emph{streams}, independent units that enable parallel encoding and decoding, containing $S=B/L$ symbols at a fixed budget of $F$ bits. The sizes are tuned based on the expected size of the encoded stream. The mapping from tensor positions to streams is cyclic. Consecutive elements land in different streams so streams do not hold a full row or column of a tensor, which may be strongly correlated (e.g. a neuron or batch may output all constant values).

A fixed budget cannot cover every stream, because code lengths vary. If a stream's codes exceed $F$, LCT writes an \emph{overflow tail}. The remaining exponents are stored uncompressed. A marker indicates where the tail starts, if any. 

\subsection{User-selected models and budgets}
\label{sec:settings}

The users have several options to choose from when selecting $\gQ$. By specifying the distributions a priori, LCT avoids computing Huffman tables at runtime, improving compression performance, while still being applicable to different scenarios.

\paragraph{Distribution.}
LCT provides a set of default codebook distributions, including a power-law model with an exponential body and a power-law tail, based on Nemotron-H activations. In addition, several other parameters may be specified. A \emph{shift} parameter allows for distributions shifted by an additive constant. \emph{Zero inflation} specifies a percentage of the distribution that is set to zero. Together with exponent shifting, these handle the typical distributions encountered during training. The user may specify additional exponent distributions if required in terms of a probability density table.

\paragraph{Storage budget.}
The storage budget is specified as \texttt{CLEAN} (default), \texttt{SPARSE}, \texttt{MEDIUM}, or \texttt{HIGH}. Larger budgets increase the base storage used, but may decrease the amount of overflow for noisy or out-of-distribution tensors. The \texttt{SPARSE} setting provides the smallest budget and is intended for sparse ReLU activations. Table~\ref{tab:geometry} lists the stream configurations and budgets.

\subsection{Standard encoding and decoding}
\label{sec:standard-codec}

\paragraph{Encoding.}
The input is first flattened and padded if needed. A small GPU kernel samples up to 4096 pseudo-random elements to estimate an integer exponent center (to avoid batch correlations) and a second kernel shifts the selected encoding table to this center. The main encoding kernel assigns one program per block and one lane per stream. It reads BF16 values using the cyclic mapping, extracts and stores their unchanged sign-mantissa bytes, and packs the exponent codes into the fixed budget. Each stream records where its fixed budget is exceeded. Subsequent kernels count overflowing streams and tail bytes, compact the stream metadata, and copy raw exponent tails from the input into private or shared storage (Section~\ref{sec:buffer}), with early exit if there is no overflow. 

\paragraph{Decoding.}
The decoding table is first shifted by the stored exponent center. The main decoding kernel traverses each stream independently, looking up a prefix, with a lookup table for decoding. Escaped symbols additionally consume their literal bits. It combines each recovered exponent with its sign--fraction byte and writes the reconstructed BF16 bits. A separate overflow kernel reads the raw tails and overwrites the corresponding output positions.

\subsection{Fused pointwise operators}
\label{sec:execution}
\label{sec:training-operators}
The optimizer requires several pointwise operations, such as adding and multiplying compressed tensors. For example, the SGD update is $W_{t+1}=W_t-\eta M_t$,
where $W_t$ and $W_{t+1}$ are compressed weights, $M_t$ is compressed momentum, and $\eta$ is the learning rate. A naive implementation would require decompressing $W_t$ and $M_t$, dense arithmetic, then recompressing the result.
For an $N$-element tensor, each encode-decode cycle uses roughly $4N$ bytes of memory traffic. A naive SGD update would require approximately $16N$ bytes of memory traffic.

We introduce fused operations for pointwise operations required during training. Each operation consists of two stages. First, a kernel takes in the compressed tensors, computes the output and saves the exponent and sign--fraction byte to separate arrays. Next, the exponents are compressed. This substantially reduces memory traffic. Our fused implementation requires $\sim8N$ bytes (the first kernel requires $\sim6N$ bytes and exponent compression requires $\sim2N$).

A selection of kernels is provided covering all the use cases for typical training with compressed and dense outputs. GEMMs still reconstruct compressed operands as dense inputs, since we did not find significant speedup with fused encode/decode GEMMs. Table~\ref{tab:training-operators} lists all operators and training uses.
Appendix~\ref{app:kernels} describes the codec and pointwise kernels, including stream parallelism, table lookups, register-based bit packing, and overflow handling.

\subsection{Shared GPU buffer}
\label{sec:buffer}
Our compressed tensors have variable sizes. The sizes are not known at compile time, so allocating exact-sized buffers at runtime introduces host synchronization. This slows down execution, particularly for our fused pointwise kernels. To avoid this, we introduce the shared \texttt{OverflowBuffer}. The OverflowBuffer reserves a large block of device memory, and allocates variable-size tensors without host synchronization. All of our kernels work with the buffer by specifying storage location through the buffer's address and tensor-specific offset. This is an optional feature.

The OverflowBuffer implements a minimal memory management framework entirely in Triton, including allocating tensors, releasing tensors, alignment, and merging adjacent free regions to reduce fragmentation. Operations are performed asynchronously. Appendix~\ref{app:overflow_buffer} provides more details.

\subsection{PyTorch training integration}
\label{sec:experiments}

Compressed tensors are represented as a \texttt{CompressedTensor} class, storing the compressed exponents, overflow tails, sign--fraction bytes, and metadata. Custom \texttt{torch.autograd.Function} classes integrate support for LCT into training. These functions compress saved tensors in the forward pass and reconstruct them in the backward pass, maintaining numerical equivalence. Custom compressed optimizers support compressed momentum buffers using the fused pointwise operators. Implementations of SGD-M, Adam \citep{kingma2015adam}, and Muon \citep{muon} are provided. Listing~\ref{lst:ffn-autograd} shows an example of a feed-forward layer implementation, including activation compression in forward and reconstruction for the backward matrix products.

\section{Analysis of LCT}

For exponent symbols drawn from $\gP$, an ideal entropy coder using model $\gQ$ has expected cost $H(\gP)+D_{\mathrm{KL}}(\gP\Vert\gQ)$ bits per exponent. Storing the sign--fraction byte unchanged adds eight bits per BF16 value. Excluding overflow, LCT uses finite prefix codes with mean length $\bar\ell=\sum_t p_t\ell_\gQ(t)$, including escape literals.
The expected code length per BF16 value is
\begin{equation}
 8+\bar\ell
 =8+H(\gP)+D_{\mathrm{KL}}(\gP\Vert\widetilde{\gQ})-\log_2 Z_\gQ
 \label{eq:lct-expected-code-rate}
\end{equation}
where $\widetilde{\gQ}=(\widetilde q_t)_t$ is the probability distribution represented by the resulting code lengths, rather than the design distribution $\gQ$, and $0<Z_\gQ\leq1$ is the codebook normalization constant. LCT's cost consists of source entropy, mismatch to the code-induced distribution $\widetilde{\gQ}$, and unused prefix-code capacity $-\log_2 Z_\gQ$. A full derivation and definitions of these quantities are given in Appendix~\ref{app:theory}.

\paragraph{Theoretical compression performance.}
Table~\ref{tab:exponent-cross-entropy} compares the ideal cross entropy and average exponent storage for LCT for each source--model pair. Under mismatch LCT's bounded codes can deliver better performance than entropy coding. This is particularly pronounced for a power-law source encoded with a Gaussian model, whose much lighter tail assigns extremely small probabilities to large exponents. Appendix~\ref{app:entropy-matrix} explains the calculation and this tail effect.

\begin{table}[t]
\centering
\setlength{\tabcolsep}{3pt}
\renewcommand{\arraystretch}{1.15}
\caption{Ideal cross entropy (left) and allocated LCT exponent storage, (right), in bits per exponent. Source distributions are in rows and encoding models in columns. G: Gaussian; L: Laplace; PL: power-law; $\Gamma$: Gamma. Bold entries use matched models.}
\label{tab:exponent-cross-entropy}
\small
\begin{minipage}[t]{0.49\linewidth}
\centering
Ideal cross entropy (bits/exponent)\par\smallskip
\begin{tabular}{@{}lrrrr@{}}
\hline
$\gP\backslash\gQ$ & G & L & PL & $\Gamma$ \\
\hline\noalign{\vskip 3pt}
Gaussian & $\mathbf{2.545}$ & $2.630$ & $2.637$ & $3.007$ \\[3pt]
Laplace & $3.039$ & $\mathbf{2.832}$ & $2.844$ & $3.062$ \\[3pt]
Power-law & $8\!\times\!10^{7}$ & $2.976$ & $\mathbf{2.891}$ & $3.118$ \\[3pt]
Gamma & $5.644$ & $3.417$ & $3.307$ & $\mathbf{3.015}$ \\[3pt]
\hline
\end{tabular}

\end{minipage}\hfill
\begin{minipage}[t]{0.49\linewidth}
\centering
Allocated LCT storage (bits/exponent)\par\smallskip
\begin{tabular}{@{}lrrrr@{}}
\hline
$\gP\backslash\gQ$ & G & L & PL & $\Gamma$ \\
\hline\noalign{\vskip 3pt}
Gaussian & $\mathbf{3.000}$ & $3.125$ & $3.125$ & $3.692$ \\[3pt]
Laplace & $3.636$ & $\mathbf{3.126}$ & $3.126$ & $3.725$ \\[3pt]
Power-law & $3.942$ & $3.176$ & $\mathbf{3.133}$ & $3.828$ \\[3pt]
Gamma & $6.554$ & $4.385$ & $3.972$ & $\mathbf{3.235}$ \\[3pt]
\hline
\end{tabular}

\end{minipage}
\end{table}

\paragraph{Empirical compression performance.}
Table~\ref{tab:nemotron-codec-compression} reports measured compression ratios for Nemotron-H-8B parameters in attention, MLP, and Mamba projection layers using our power-law, Gaussian, Laplace, and Gamma codecs. The heavy-tailed weights mean the Laplace and Power-Law distributions give the best compression ratio, close to the theoretical optimum. 

\begin{table}[t]
\centering
\caption{Average exponent storage (bits/exponent, left) and measured compression ratios (right) for Nemotron-H-8B parameters using four LCT codecs. Bold entries indicate the best results in each parameter group.}
\label{tab:nemotron-codec-compression}
\small
\setlength{\tabcolsep}{5pt}
\begin{tabular}{@{}lrrrrrr@{}}
\hline
& \multicolumn{3}{c}{Average bits per exponent} & \multicolumn{3}{c}{Compression ratio ($\times$)} \\
\cline{2-4}\cline{5-7}
\noalign{\vskip 3pt}
Codec & Attention & MLP & \shortstack{Mamba\\projection} & Attention & MLP & \shortstack{Mamba\\projection} \\[1pt]
\hline
Power-Law & \textbf{3.1250} & 3.4041 & \textbf{3.1304} & \textbf{1.4382} & 1.4030 & \textbf{1.4375} \\
Gaussian & 3.4984 & 4.1793 & 3.8003 & 1.3915 & 1.3137 & 1.3559 \\
Laplace & \textbf{3.1250} & \textbf{3.4025} & 3.1312 & \textbf{1.4382} & \textbf{1.4032} & 1.4374 \\
Gamma & 3.2668 & 3.4416 & 3.1312 & 1.4201 & 1.3984 & 1.4374 \\
\hline
\end{tabular}
\end{table}

\subsection{Adaptation to magnitude scale}
\label{sec:scale-adaptation}

Tensor magnitudes change during training which LCT handles by shifting the codebook to match the exponent center at runtime. The following propositions explain why this compensates for rescaling.

\begin{proposition}[Rescaling mostly shifts exponents]
\label{prop:exponent-scaling}
Ignoring rounding and finite-range effects, multiplying finite, nonzero values by a finite, nonzero constant shifts their exponents by a common integer, with at most a one-bin difference remaining after correction. For power-of-two scaling, the shift correction is exact.
\end{proposition}

\begin{proposition}[Rare boundary effects give a small correction]
\label{prop:boundary-entropy}
For the same rescaling, let $A$ and $B$ be the stored BF16 exponent bytes before and after scaling. If rounding or finite-range effects alter either ideal exponent label with probability $\varepsilon$, then, with $h_2$ as the binary entropy function,
\begin{equation}
 |H(B)-H(A)|\leq1+7\varepsilon+h_2(\varepsilon),
 \label{eq:scale-boundary-bound}
\end{equation}
\end{proposition}

Appendix~\ref{app:scale-transform} proves both propositions. Proposition~\ref{prop:exponent-scaling} explains how exponent centering corrects for changes in scale. Proposition~\ref{prop:boundary-entropy} bounds the entropy change by one bit plus a small correction. Together, these propositions explain why shifting the codebook can adapt LCT to different input scales at runtime. 

\subsection{Allocated rate and scaling with tensor size}

For a tensor of $N$ elements comprising full blocks, let $S$ be the symbols per stream, $F$ the reserved bits per stream, $T$ the number of raw tail exponents, and $K$ the number of overflowing streams. Each tail exponent occupies one byte and each affected stream requires nine metadata bytes. The core rate is
\begin{equation}
 b_{\mathrm{LCT}}=8+\frac{F}{S}+8\frac{T}{N}+72\frac{K}{N},
 \qquad \rho_{\mathrm{LCT}}=\frac{16}{b_{\mathrm{LCT}}}.
 \label{eq:allocated-ratio}
\end{equation}
The first two terms are the regular payload; the last two account for raw overflow tails and overflow metadata. A smaller $F$ reduces the fixed payload but can increase the number of overflowing streams $K$ and raw tail length $T$, reducing the compression ratio. This assumes full blocks and excludes padding and allocation overhead. Appendix~\ref{app:memory} gives a more detailed analysis.

\section{Experiments} \label{sec:benchmarks}

\subsection{Encoding and decoding bandwidth}
\label{sec:codec-benchmarks}
First, we evaluate LCT's encoding and decoding performance and compare it with existing methods.
Table~\ref{tab:codec-bandwidth} compares the time to encode and decode a $1\,\mathrm{GiB}$ Gaussian tensor using LCT, SplitZip, DFloat11, ZipNN, and three nvCOMP codecs, LZ4, ANS, and Bitcomp. Appendix~\ref{app:baseline-evaluation} gives implementation details.

LCT encodes fastest, with $2.38\times$ the throughput of SplitZip, the next-fastest encoder. SplitZip decodes slightly faster, but LCT has a lower combined encode--decode time (4.967 versus 8.159\,ms) and a higher compression ratio ($1.455\times$ versus $1.333\times$). The nvCOMP LZ4 and Bitcomp configurations slightly expand the input ($0.998\times$), while nvCOMP ANS achieves $1.258\times$ compression. All three encode and decode more slowly than LCT. DFloat11 and ZipNN achieve higher compression ratios, but their encoding times are over two orders of magnitude longer than LCT's.

Next, Figure~\ref{fig:codec-throughput} evaluates the throughput for varying input sizes. LCT scales well with input size, with peak encode throughput reaching $395.5$\,GiB/s.

Table~\ref{tab:codec-gpu-comparison} in the appendix extends the comparison to NVIDIA V100, A40, and H100 GPUs. LCT remains the fastest encoder, while SplitZip remains the fastest decoder.

\begin{table}[t]
\centering
\caption{Time taken to encode and decode a $1\,\mathrm{GiB}$ tensor. Standard errors in the times are shown. Compression ratio is original size divided by compressed size. The fastest method is marked in bold.}
\label{tab:codec-bandwidth}
\small
\setlength{\tabcolsep}{9pt}
\begin{tabular}{lrrr}
\hline
Method & Encode / ms & Decode / ms & Compression ratio \\
\hline
nvCOMP LZ4 & $484.584\pm0.646$ & $8.761\pm0.035$ & $0.998\times$ \\
nvCOMP Bitcomp & $40.413\pm0.585$ & $4.193\pm0.035$ & $0.998\times$ \\
nvCOMP ANS & $53.028\pm0.489$ & $4.349\pm0.059$ & $1.258\times$ \\
DFloat11 & $2574.000\pm0.671$ & $4.475\pm0.014$ & $1.483\times$ \\
ZipNN & $845.400\pm0.841$ & $199.300\pm0.619$ & $1.511\times$ \\
SplitZip & $5.850\pm0.016$ & $\mathbf{2.309\pm0.004}$ & $1.333\times$ \\
LCT & $\mathbf{2.454\pm0.005}$ & $2.513\pm0.005$ & $1.455\times$ \\
\hline
\end{tabular}
\end{table}

\subsection{Fused pointwise operators}
\label{sec:fused-benchmarks}

Table~\ref{tab:fused-benchmarks} compares fused operators with two baselines using Gaussian BF16 input tensors of 1\,GiB each. The \emph{naive} path decompresses all operands before computation and returns the requested output format. The \emph{dense} baseline is the standard PyTorch implementation, without any compression. Appendix~\ref{app:fused-evaluation} gives the protocol. For dense output, LCT is $1.66$--$1.72\times$ faster than the naive path. Operations with one compressed operand also slightly outperform the dense baseline since compression reduces the size of the tensor, and hence memory traffic. For compressed outputs, LCT is $1.64$--$2.56\times$ faster than the naive path. 

\begin{table}[t]
\centering
\caption{Fused pointwise operator timings (mean $\pm$ standard error). $C$ denotes a compressed operand, $D$ a dense operand, and $\alpha$ a scalar. Speedup compares LCT with Naive.}
\label{tab:fused-benchmarks}
\small
\setlength{\tabcolsep}{4pt}
\renewcommand{\arraystretch}{1.12}
\begin{tabular}{@{}llrrrr@{}}
\hline
Operation & Output & LCT / ms & Naive / ms & Dense / ms & Speedup \\
\hline
$C+D$ & Dense & $3.694\pm0.003$ & $6.342\pm0.015$ & $3.855\pm0.003$ & $1.717\times$ \\
$C\odot D$ & Dense & $3.691\pm0.003$ & $6.347\pm0.021$ & $3.866\pm0.014$ & $1.720\times$ \\
$\alpha C+D$ & Dense & $3.822\pm0.003$ & $6.331\pm0.008$ & $3.859\pm0.005$ & $1.656\times$ \\
$\alpha C_1+C_2$ & Dense & $5.262\pm0.004$ & $8.709\pm0.003$ & $3.864\pm0.009$ & $1.655\times$ \\
\hline
$C+D$ & Compressed & $5.058\pm0.010$ & $8.858\pm0.009$ & $3.852\pm0.003$ & $1.751\times$ \\
$C\odot D$ & Compressed & $5.006\pm0.009$ & $12.814\pm0.019$ & $3.848\pm0.003$ & $2.560\times$ \\
$\alpha C+D$ & Compressed & $5.140\pm0.010$ & $8.700\pm0.006$ & $3.868\pm0.012$ & $1.693\times$ \\
$\alpha C_1+C_2$ & Compressed & $6.768\pm0.021$ & $11.091\pm0.009$ & $3.854\pm0.003$ & $1.639\times$ \\
\hline
\end{tabular}
\end{table}

\subsection{Training NanoGPT}
\begin{table}[t]
\centering
\caption{NanoGPT training time per optimizer step and peak PyTorch-allocated GPU memory, including the shared buffer. The effective compression ratio is baseline peak memory divided by each configuration's peak memory, including all live GPU allocations.}
\label{tab:nanogpt-compression}
\small
\setlength{\tabcolsep}{4pt}
\begin{tabular}{@{}lrrr@{}}
\hline
Compressed components & Time / s & Peak memory / MiB & \shortstack{Effective\\compression ratio} \\
\hline
None & 8.411 & 16058.0 & $1.000\times$ \\
Activations & 9.013 & 12103.0 & $1.327\times$ \\
+ Weights & 9.131 & 12013.8 & $1.337\times$ \\
+ Momentum Buffers & 9.136 & 11964.6 & $1.342\times$ \\
\hline
\end{tabular}
\end{table}

A 162M-parameter transformer is trained on the FineWeb-10B dataset \citep{fineweb} for 3400 steps with and without LCT. The setup is based on Track 3 of the Modded NanoGPT project \citep{modded_nanogpt}, with optimizations to reduce the baseline memory usage; Appendix~\ref{app:modded_nanogpt} gives the details. The LCT run uses a $64$\,MiB shared OverflowBuffer. Final validation losses are $3.27978$ with LCT and $3.27981$ without LCT, within nondeterministic rounding error.

\paragraph{Ablations and checkpoint analysis.}
Table~\ref{tab:nanogpt-compression} compares compression configurations in a short run, reporting time per training step, peak memory, and effective compression ratio. Activation compression reduces peak memory from 16058 to 12103\,MiB ($24.6\%$), with a $7.2\%$ increase in time. Further compressing weights and momentum buffers saves another $138.4$\,MiB. Table~\ref{tab:nanogpt-buffer} isolates the shared OverflowBuffer: it reduces step time by $2.3\%$ at a small memory cost. Figure~\ref{fig:nanogpt-exponent-history} tracks feed-forward weight and momentum exponent distributions across checkpoints. Their magnitudes shift by roughly $2^6=64$. Table~\ref{tab:nanogpt-checkpoint-compression} reports the corresponding checkpoint compression ratios. The compression ratios are extremely consistent beyond the first few steps, reaching the $1.4382$ limit of the \texttt{CLEAN} scheme. The first steps are slightly worse because many values are zero-initialized.

\subsection{Fine-tuning larger models}
A 4B parameter Qwen3-4B model \citep{qwen3} is fine-tuned to investigate the performance on a larger model, where parameters and momentum buffers make up a more significant portion of the memory usage. Full-parameter fine-tuning in BF16 is used for all weights, activations, and momentum. Without compression, momentum and parameters use around $8$\,GiB of memory each. Full setup details are given in Appendix~\ref{app:qwen_finetune}.

\begin{table}[t]
\centering
\caption{Qwen3-4B full-parameter fine-tuning at batch size 1. Memory is peak PyTorch-allocated GPU memory, including any shared buffer, and time is per training step. Dashes indicate runs that failed to complete.}
\label{tab:qwen-finetuning}
\small
\setlength{\tabcolsep}{6pt}
\renewcommand{\arraystretch}{1.1}
\begin{tabular}{@{}rrrrr@{}}
\hline
 & \multicolumn{2}{c}{Baseline} & \multicolumn{2}{c}{LCT} \\
\cline{2-3}\cline{4-5}
Sequence length & Memory / MiB & Time / s & Memory / MiB & Time / s \\
\hline
500  & 19439.2 & 3.759 & 15141.8 & 4.058 \\
1000 & 21327.8 & 4.006 & 16262.4 & 4.273 \\
1500 & --- & --- & 17370.1 & 4.523 \\
2500 & --- & --- & 19617.6 & 5.055 \\
3500 & --- & --- & 21842.1 & 5.607 \\
\hline
\end{tabular}
\end{table}

Table~\ref{tab:qwen-finetuning} compares standard tuning and LCT across sequence lengths from 500 to 3500 tokens. At 500 and 1000 tokens, LCT reduces peak memory by $22.1\%$ and $23.8\%$, respectively, with step-time increases of $8.0\%$ and $6.7\%$. The uncompressed baseline runs out of memory for the tested sequence lengths above 1000 tokens. This extends the largest supported sequence length by $3.5\times$, demonstrating the effectiveness of compression and how it can be used to train longer sequences on the same hardware.

\section{Conclusion}
\label{sec:conclusion}

\textbf{Finding.} We presented LCT, a lossless BF16 tensor codec designed for repeated encoding and decoding during training. LCT combines bounded Huffman coding, parallel GPU execution, fused pointwise updates, and shared GPU memory management. Across the evaluated GPUs, it provides the fastest encoding among the tested methods and reduces memory usage in both training and finetuning with little performance penalty. These results show that lossless compression can serve as a practical training-memory primitive when both encoding and decoding are designed for repeated GPU execution. \textbf{Limitations and scope.} LCT currently targets BF16, relies on user-specified tensor distributions, and requires explicit integration with training operations. Its effectiveness may vary for architectures or tensor types with different exponent statistics. Future work includes decoding within GEMM kernels, improved adaptation to distribution shift, and combining lossless with lossy compression.



\subsection*{AI use statement}

Agentic LLMs were used to optimize the Triton kernels. We first manually designed a series of test harnesses that contain correctness checks and benchmarks (given in the \texttt{experiments} folder of the Supplementary materials). An Autoresearch agent was targeted at these specifications. The prompt is reproduced in \texttt{program.md}. Manual supervision was used throughout the process, with suggestions and advice given to the agent. Ideas such as the bounded Huffman scheme, OverflowBuffer and fused optimizer operators are the Authors, and given to the LLM to implement. 

Once the kernels were optimized, the rest of the code was written with less LLM assistance. This includes the PyTorch interface and autograd functions and main experimental setups.

Generative AI tools were used to assist with manuscript drafting and revision, mathematical exposition, and the preparation of figures and tables. AI assistance was also used to check for errors in the manuscript.

\subsection*{Reproducibility statement}

Appendix~\ref{app:experimental-details} describes the benchmark protocols and training configurations. Algorithm~\ref{alg:codebook} specifies codebook construction, with its optimality argument in Appendix~\ref{app:tables}. Appendix~\ref{app:kernels} details the GPU implementation, and Listing~\ref{lst:ffn-autograd} illustrates integration with automatic differentiation. The assumptions and derivations underlying the compression analysis are provided in Appendix~\ref{app:theory}, while Appendix~\ref{app:measurements} describes the weight and activation measurements.

The supplementary materials contain our codebase, including the LCT package, Autoresearch prompt and code for the experiments.

\subsection*{Ethics statement}
This work studies lossless compression of model tensors and evaluates it using synthetic data, the existing FineWeb dataset, and pretrained models. No human participants were recruited and no new personal data were collected. Use of web-derived training data remains subject to privacy, copyright, and bias considerations. The main impact of this work is lowering memory requirements, making larger training workloads more accessible. The energy impact of compression depends on the full training workload, including its runtime overhead.
\clearpage
\bibliography{references}
\bibliographystyle{plainnat}
\clearpage
\appendix
\section{Experimental Details}
\label{app:experimental-details}
The main experiments are conducted on an NVIDIA RTX 3090 using PyTorch 2.14 \citep{pytorch} with Triton 3.8 \citep{triton}. The evaluation across GPU models in Table~\ref{tab:codec-gpu-comparison} used the described GPUs, with an older PyTorch 2.11 setup for compatibility with the older GPUs.

\paragraph{Training measurements.}
For the NanoGPT and Qwen timing and memory comparisons, we discard three warm-up optimizer steps and measure the following five steps. We report the mean step time and the largest peak memory across those five steps. Wall-clock timing is synchronized with CUDA before and after each step and includes batch preparation, forward and backward computation, and optimizer updates. Model loading and setup are outside the timed region; initial compilation, where enabled, is performed during warm-up. Validation and checkpoint saving are excluded from these measurements.

Peak memory is measured with \texttt{torch.cuda.max\_memory\_allocated()}, resetting the peak counter before each step. This includes persistent tensors, temporary allocations, and the full shared overflow buffer when used. Values are reported in MiB and measure allocated tensor memory rather than the CUDA allocator's reserved memory.

\subsection{Baseline codec evaluation}
\label{app:baseline-evaluation}

We compare LCT, SplitZip, DFloat11, ZipNN, and nvCOMP with LZ4, Bitcomp, and ANS on the encoding and decoding task in Section~\ref{sec:codec-benchmarks}. The input contains $2^{29}$ BF16 elements ($1\,\mathrm{GiB}$), sampled on the GPU from $\mathcal{N}(0,2^2)$ in FP32 and then converted to BF16, using random seed 0.

\paragraph{Codec configuration.}
SplitZip, DFloat11, and ZipNN use the official implementations. SplitZip's codebook is calibrated on a separate tensor of the same size and distribution. The calibration time is excluded from timing since it is intended to be done once before use. DFloat11 uses a native C encoder on the CPU and the upstream CUDA decoder. Its encoding time includes histogram and codebook construction, transfer of the input to the CPU, and transfer of the compressed buffers back to the GPU. ZipNN uses Huffman coding with eight CPU threads. The encoding time includes the input transfer to the CPU, and decoding includes transfer of the reconstructed tensor to the GPU.

LCT uses the Gaussian model, with a shared overflow arena allocated before timing. The nvCOMP methods all use the default configuration. Timings include any setup and codec configuration done as part of the library.

\paragraph{Timing and storage.}
Each codec performs three warm-up encode--decode cycles, followed by 50 measurements of each operation. Encoding and decoding are timed separately with a wall-clock timer and CUDA synchronization at their boundaries. Input generation, codec initialization, and output cleanup are outside the timed regions; allocations and transfers performed within each codec call are included. The reported means therefore measure complete codec calls, including host overhead. Timing uncertainties are standard errors of the mean, computed as the sample standard deviation divided by $\sqrt{50}$. Compression ratios divide the original tensor size by the compressed representation size. Baseline sizes include the encoded payload and decoding metadata; LCT uses the storage charged to the compressed tensor, including its overflow allocation, rather than the full capacity of the shared arena.

\paragraph{Additional GPU evaluations.}
Table~\ref{tab:codec-gpu-comparison} reports the codec comparison using a 1\,GiB Gaussian tensor on NVIDIA V100 and A40 and a 4\,GiB Gaussian tensor on NVIDIA H100. LCT has the lowest encoding time among the reported methods on all three GPUs, while SplitZip has the lowest decoding time. LCT's storage geometry was changed on the V100 and H100, using 32,768 elements per block rather than 65,536, with stream configurations changed accordingly to match the compression ratio. This allows for better parallelism on the more powerful GPUs. Our V100 setup does not support nvCOMP methods.

\begin{table}[t]
\centering
\caption{Encoding and decoding Gaussian tensors of $1\,\mathrm{GiB}$ on NVIDIA V100 and A40, and $4\,\mathrm{GiB}$ on NVIDIA H100. Times are mean $\pm$ standard error of the mean; lower is better. Dashes indicate unreported results. Bold entries mark the lowest mean for each operation on each GPU.}
\label{tab:codec-gpu-comparison}
\scriptsize
\setlength{\tabcolsep}{2pt}
\renewcommand{\arraystretch}{1.15}
\begin{tabular}{@{}lrr@{\hspace{8pt}}rr@{\hspace{8pt}}rr@{}}
\hline
 & \multicolumn{2}{c}{NVIDIA V100 (1 GiB)} & \multicolumn{2}{c}{NVIDIA A40 (1 GiB)} & \multicolumn{2}{c}{NVIDIA H100 (4 GiB)} \\
\cline{2-3}\cline{4-5}\cline{6-7}
Method & Encode / ms & Decode / ms & Encode / ms & Decode / ms & Encode / ms & Decode / ms \\
\hline
nvCOMP LZ4 & --- & --- & $629.2\pm0.242$ & $13.64\pm0.138$ & $500.1\pm12.33$ & $16.63\pm11.79$ \\
nvCOMP Bitcomp & --- & --- & $46.43\pm0.134$ & $5.354\pm0.002$ & $25.74\pm4.997$ & $4.771\pm4.771$ \\
nvCOMP ANS & --- & --- & $61.75\pm0.008$ & $5.341\pm0.002$ & $34.6\pm1.408$ & $6.349\pm0.081$ \\
DFloat11 & $5798\pm196.0$ & $6.271\pm0.103$ & $28430\pm2.245$ & $4.750\pm0.007$ & $17550\pm208.1$ & $11.18\pm0.022$ \\
ZipNN & $2504\pm179.6$ & $470.2\pm42.78$ & $1250\pm1.592$ & $529.5\pm1.053$ & $7001\pm329.5$ & $1290\pm30.6$ \\
SplitZip & $5.873\pm0.586$ & $\mathbf{2.461\pm0.072}$ & $7.758\pm0.001$ & $\mathbf{3.580\pm0.002}$ & $9.039\pm0.008$ & $\mathbf{2.652\pm0.009}$ \\
LCT & $\mathbf{3.317\pm0.073}$ & $3.479\pm0.078$ & $\mathbf{3.593\pm0.001}$ & $3.922\pm0.002$ & $\mathbf{3.748\pm0.006}$ & $3.904\pm0.007$ \\
\hline
\end{tabular}
\end{table}

\paragraph{Scaling with input size.}
Figure~\ref{fig:codec-throughput} compares encoding and decoding throughput for inputs from 16 to 2048\,MiB. LCT's encoding throughput exceeds SplitZip's from 64\,MiB onward, reaching approximately 396\,GiB/s at 2048\,MiB. Their decoding throughputs are similar for larger inputs, both reaching approximately 385\,GiB/s at 2048\,MiB.

\begin{figure}[t]
 \centering
 \includegraphics[width=0.8\linewidth]{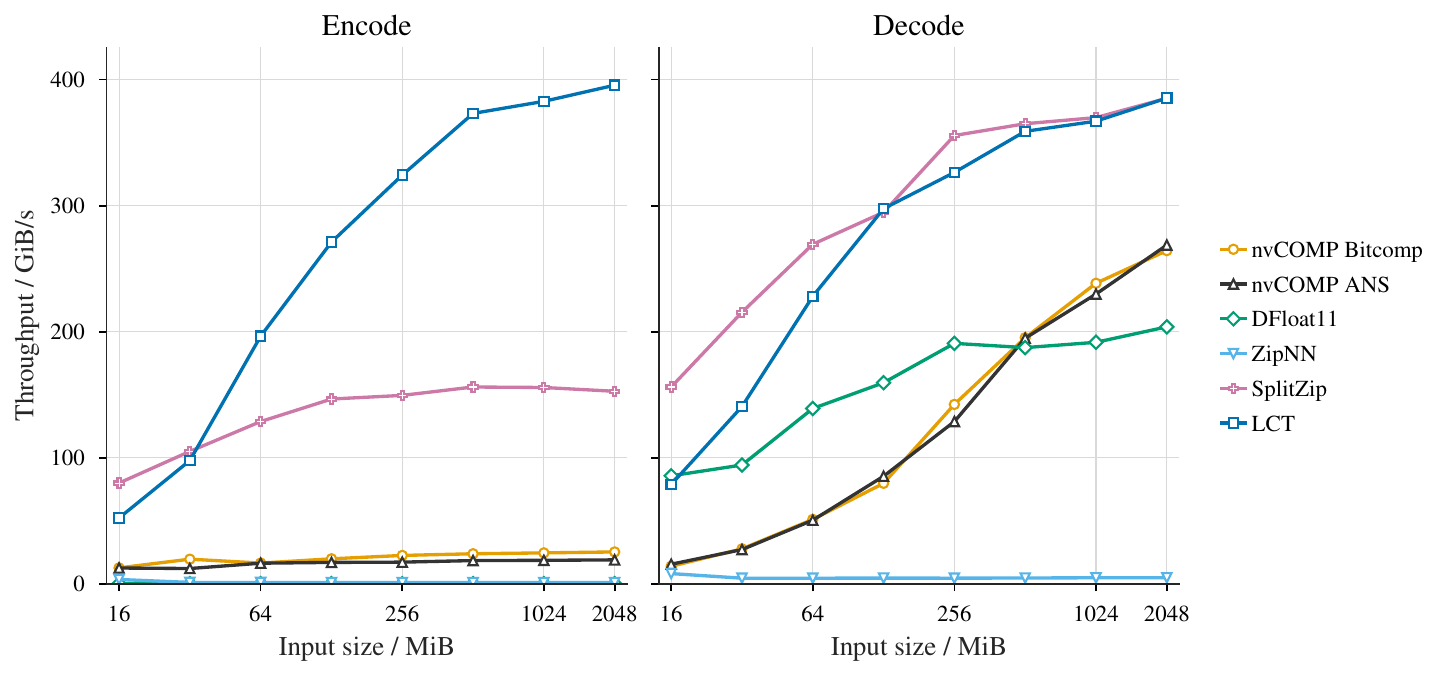}
 \caption{Encoding (left) and decoding (right) throughput versus uncompressed input size. Throughput is the uncompressed size in GiB divided by the mean time in seconds. LZ4 and Cascaded are omitted.}
 \label{fig:codec-throughput}
\end{figure}

\subsection{Fused pointwise operator evaluation}
\label{app:fused-evaluation}

Each pointwise operation was benchmarked. Each input contains $2^{29}$ BF16 elements ($1\,\mathrm{GiB}$). Addition operands are sampled on the GPU from $\mathcal{N}(0,2^2)$. For multiplication, the second operand instead follows $\mathcal{N}(1,0.01^2)$ with seed 2, to keep the multiplier concentrated near one. Input and output compression use the Gaussian codebook, \texttt{CLEAN} storage budget, and no added zero inflation or bias. The compressed paths use preallocated OverflowBuffers.

Before timing, each variant is executed once for compilation and bitwise comparison against a dense reference. Scaled-add references are computed in FP32 and rounded to BF16. Each variant then receives three warm-up calls and 50 timed calls. We use \texttt{time.perf\_counter}, with CUDA synchronization immediately before and after each call. The order of the three variants is rotated and reversed across iterations to reduce execution-order bias. Input generation, initial compression, buffer construction, compilation, validation, and output release are outside the timed region.

\subsection{Training Modded NanoGPT} \label{app:modded_nanogpt}
The base model architecture, data, hyperparameters, and training code were taken from the baseline of Phase 3 of the Modded NanoGPT project. The code was then modified to reduce memory consumption without changing behavior. BF16 precision was used for all activations, momentum and weights. Activation checkpointing was applied to the RMSNorm layers, and the final output head was computed in chunks. Optimizer fusion was also used \citep{optimizer_fusion}. These do not affect the training process or model outputs, but significantly reduce the memory consumption. This is done to ensure a well-optimized baseline for a fair comparison. The Muon optimizer with 12 Newton--Schulz iterations is used, along with a batch size of 256 accumulated over four microbatches of 64 sequences, each of length 1024 tokens. The timing comparison measures a complete optimizer step, including all four microbatches.

LCT is applied to all matrix weights, activations, and momentum buffers. 1-D biases are excluded. Custom functions which compress cached activations are used for the feed-forward layers, attention block, and embedding block. Each type of compressed cached activation is given in Table~\ref{tab:nanogpt-activations}. Activations make up the majority of the storage; model and optimizer states take up $767.88$\,MiB.

\begin{table}[t]
\centering
\caption{Cached activations compressed in the Modded NanoGPT experiment. Sizes are the dense BF16 storage before compression for a batch of 64.}
\label{tab:nanogpt-activations}
\small
\setlength{\tabcolsep}{3pt}
\renewcommand{\arraystretch}{1.15}
\begin{tabularx}{\linewidth}{@{}>{\raggedright\arraybackslash}Xrlrrl@{}}
\hline
Activation & Count & Shape & Each / MiB & Total / MiB & Budget \\
\hline
Embedding output before RMSNorm & 1 & $[64,1024,768]$ & 96 & 96 & CLEAN \\
Input to RMSNorm and QKV projections & 12 & $[64,1024,768]$ & 96 & 1152 & CLEAN \\
Q and K before per-head RMSNorm & 24 & $[64,1024,6,128]$ & 96 & 2304 & CLEAN \\
FlashAttention Q/K/V/output caches & 48 & $[64,1024,6,128]$ & 96 & 4608 & CLEAN \\
FFN input before RMSNorm and first linear & 12 & $[64,1024,768]$ & 96 & 1152 & CLEAN \\
FFN ReLU output before squaring & 12 & $[64,1024,3072]$ & 384 & 4608 & SPARSE \\
\hline
Total & 109 & & & 13920 & \\
\hline
\end{tabularx}
\end{table}

In order for \texttt{torch.compile} to work, the code is carefully structured so critical parts are compiled, while graph breaks are used around LCT operations, since compiling LCT does not lead to speedups.

\paragraph{Effect of the shared buffer.}
Table~\ref{tab:nanogpt-buffer} compares NanoGPT training with and without the shared OverflowBuffer. Using the buffer reduces optimizer-step time from 9.353 to 9.136\,s ($2.3\%$), while increasing peak allocated memory by 51.7\,MiB.

\begin{table}[t]
\centering
\caption{Effect of the shared OverflowBuffer on NanoGPT training with compressed activations, weights, and momentum buffers. Time is per optimizer step; peak PyTorch-allocated GPU memory includes the shared buffer when enabled.}
\label{tab:nanogpt-buffer}
\small
\begin{tabular}{@{}lrr@{}}
\hline
Setup & Time / s & Peak memory / MiB \\
\hline
With buffer & 9.136 & 11964.6 \\
Without buffer & 9.353 & 11912.9 \\
\hline
\end{tabular}
\end{table}

\begin{table}[htbp]
\centering
\caption{Weight and momentum compression ratios at NanoGPT training checkpoints. Ratios are uncompressed size divided by compressed size. The dash denotes an unreported value. * The first momentum step uses a compression distribution with a higher zero fraction.}
\label{tab:nanogpt-checkpoint-compression}
\small
\setlength{\tabcolsep}{12pt}
\begin{tabular}{rrr}
\hline
& \multicolumn{2}{c}{Compression ratio ($\times$)} \\
Step & Weights & Momentum \\
\hline
0    & 1.2240 & --- \\
1    & 1.2240 & 1.4382* \\
2    & 1.4382 & 1.2426 \\
5    & 1.4382 & 1.4116 \\
10   & 1.4382 & 1.4043 \\
300  & 1.4382 & 1.4332 \\
900  & 1.4382 & 1.4379 \\
1500 & 1.4382 & 1.4382 \\
2100 & 1.4382 & 1.4382 \\
2700 & 1.4382 & 1.4382 \\
3350 & 1.4382 & 1.4382 \\
\hline
\end{tabular}
\end{table}

\paragraph{Feed-forward weight and momentum distributions during training.}
Figure~\ref{fig:nanogpt-exponent-history} shows how exponent distributions of feed-forward momentum buffers (left) and weights (right) change during NanoGPT training. Both plots cover the two weight matrices in each feed-forward block. Each curve averages normalized histograms with equal weight per tensor, rather than per element. Probabilities use a logarithmic scale, and exponents have the BF16 bias removed. Exact zeros are shown separately, with their mean fractions in the legends; normalization includes all elements. The early momentum distributions lose their initial zero mass, while the weight distributions shift toward larger exponents at later checkpoints.

\begin{figure}[tb]
 \centering
 \begin{minipage}[t]{0.49\linewidth}
  \vspace{0pt}
  \includegraphics[width=\linewidth]{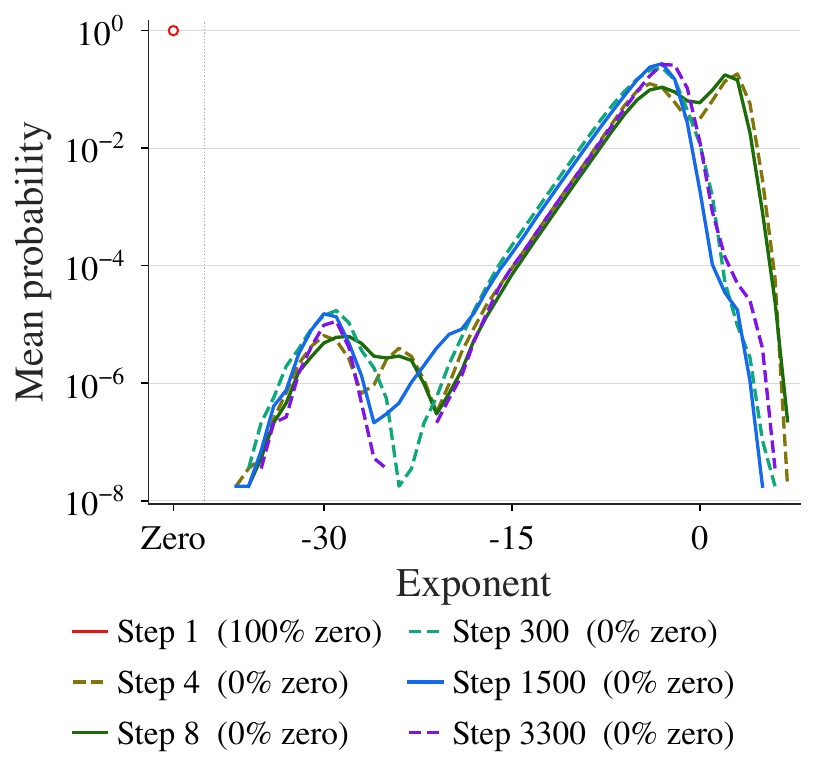}
 \end{minipage}\hfill
 \begin{minipage}[t]{0.49\linewidth}
  \vspace{0pt}
  \includegraphics[width=\linewidth]{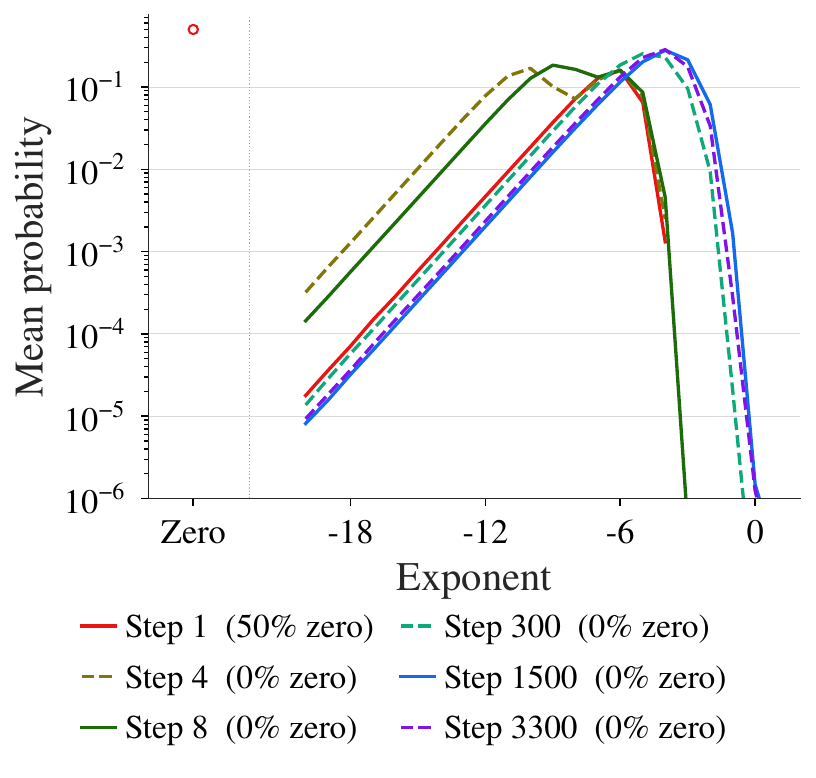}
 \end{minipage}
 \caption{NanoGPT feed-forward exponent distributions during training: Muon momentum (left) and weights (right) at steps 1, 4, 8, 300, 1500, and 3300. Weight exponents below $-20$ are omitted from view.}
 \label{fig:nanogpt-exponent-history}
\end{figure}
\FloatBarrier 

\subsection{Fine-tuning Qwen3-4B} \label{app:qwen_finetune}
The Qwen3-4B-Base \citep{qwen3} model is used. The peak memory and time per step are investigated here. Since our method is lossless, these results are applicable to other fine tuning setups, and can be used on top of methods such as LoRA \citep{lora}.

As in the Modded NanoGPT project, we fuse the RMSNorm layers, chunk the output head, and use optimizer fusion to reduce memory consumption without loss to ensure the baseline is well optimized. The timing test uses FineWeb training inputs with a batch size of 1 and varies the sequence length from 500 to 3500 tokens. Each optimizer step processes one sequence, without gradient accumulation. The Muon optimizer is used for hidden matrices with 5 Newton--Schulz iterations.

\FloatBarrier 
\section{Representation and Implementation Details}
\label{app:implementation}

\subsection{Distribution models and table construction}
\label{app:tables}

The implementation provides a set of default codebook distributions, one for each tensor role. Tables are cached by distribution family and parameters, independently of stream geometry.

Model probabilities come from integrating the magnitude density over exponent intervals: $[2^q,2^{q+1})$ for an interior exponent index $q=e-127$, with the endpoint bins extending to zero and infinity. These probabilities guide code assignment rather than certify a match to the tensor's empirical statistics. At compression time, up to 4096 stratified, jittered samples estimate a tensor-specific center $c$ from the mean of $e_i-127$, rounded and clamped to $[-128,127]$; the center is stored with the tensor and reused during decoding.

The centered symbol mapping reserves table index zero for exponent byte zero. With $d=(e-127-c)\bmod256$ and $z=(-127-c)\bmod256$,
\begin{equation}
 t(e;c)=
 \begin{cases}
 0, & e=0,\\
 d+1, & e\ne0 \text{ and } d<z,\\
 d, & \text{otherwise}.
 \end{cases}
 \label{eq:centered-symbol}
\end{equation}
This is a permutation of the 256 exponent bytes. The reserved symbol covers both exact zeros and subnormals, whose distinction is preserved by the side byte.

To select an escape cutoff, symbols are sorted by model probability and each of the 256 candidates, including no escape, is optimized with Package-Merge under the prefix cap $L_{\max}=\min(10,8)=8$ \citep{bounded_huffman}, then scored by Equation~\ref{eq:escape-search}. Because the literal cost is constant per candidate, minimizing weighted prefix length also minimizes expected cost. Escapes matter under this cap: eight-bit prefixes force all 256 direct codewords to length eight by Kraft's inequality, so grouping rare symbols under one escape leaf frees prefix space for frequent ones. Since the no-escape code is included, the selected model-expected exponent length is at most eight bits. The search is optimal within the class of canonical codes that use a single escape leaf, an eight-bit literal, and the prefix cap $L_{\max}$; we prove this below. Algorithm~\ref{alg:codebook} states the procedure.

\begin{algorithm}[t]
\caption{Codebook generation for a distribution $\gQ$}
\label{alg:codebook}
\begin{algorithmic}[1]
\REQUIRE model $\gQ=(q_t)$ over $n=256$ exponent symbols, prefix cap $L_{\max}=8$, literal width $b=8$
\ENSURE encode table $E$, decode table $D$, escape length
\STATE integrate $\gQ$ over exponent bins and apply zero mass to obtain $\mathbf{p}=(p_t)$
\STATE apply the centered-symbol permutation of Equation~\ref{eq:centered-symbol}, reserving index $0$ for exponent byte zero
\STATE sort so $p_1\geq\cdots\geq p_n$; let $S_k=\sum_{j=n-k+1}^{n}p_j$ be the mass of the $k$ rarest symbols
\STATE $\mathbf{L}\leftarrow\textsc{Package-Merge}(\mathbf{p},L_{\max})$, $c^\star\leftarrow\sum_t p_t L_t$, $k^\star\leftarrow0$ \COMMENT{no-escape candidate}
\STATE $\mathbf{L}^\star\leftarrow\mathbf{L}$
\FOR{$k=1$ to $n-1$}
  \STATE $\mathbf{w}\leftarrow(p_1,\ldots,p_{n-k},\,S_k)$ \COMMENT{replace the tail by one escape leaf}
  \STATE $\mathbf{L}\leftarrow\textsc{Package-Merge}(\mathbf{w},L_{\max})$
  \STATE $c\leftarrow\sum_t w_tL_t+bS_k$ \COMMENT{add literal cost for escaped mass}
  \IF{$c<c^\star$}
    \STATE $c^\star\leftarrow c$, $k^\star\leftarrow k$, $\mathbf{L}^\star\leftarrow\mathbf{L}$
  \ENDIF
\ENDFOR
\STATE assign canonical codes using $\mathbf{L}^\star$, including an escape leaf only if $k^\star>0$
\FOR{each escaped symbol $t$}
  \STATE $\text{code}_t\leftarrow(\text{esc\_code}\ll b)\;|\;\mathrm{index}_b(t)$, \quad $\ell_t\leftarrow L^\star_{\mathrm{esc}}+b$
\ENDFOR
\STATE build $E[256]$ from $(\text{code}_t,\ell_t)$ and $D[2^{L_{\max}}]$ by expanding each prefix
\RETURN $E$, $D$, and $L^\star_{\mathrm{esc}}$ (or $0$ if $k^\star=0$)
\end{algorithmic}
\end{algorithm}

\paragraph{Optimality.}
Sort the symbol probabilities as $p_1\geq\cdots\geq p_n$, with $n=256$. Two reductions pin down the search. First, one escape leaf suffices: given several, reassigning every escaped symbol to the shortest leaf $\ell$ preserves Kraft's inequality, since $2^{-\ell}\leq\sum_i2^{-\ell_i}$, and cannot raise the escape cost $\sum_im_i\ell_i$, since $\ell_i\geq\ell$ for all $i$. Second, the escaped set is a probability suffix: exchanging an escaped symbol $i$ with a directly coded symbol $j$ of probability $p_j\leq p_i$ reduces the cost by $(p_i-p_j)(\ell_{\mathrm{esc}}+8-\ell_j)\geq0$, because $\ell_{\mathrm{esc}}+8\geq9$ while $\ell_j\leq L_{\max}=8$. An optimal code therefore escapes exactly the $k$ least probable symbols for some $k$, and it suffices to test these probability-ranked cutoffs. For fixed $k$ the cost is $\sum_{t\leq n-k}p_t\ell_t+S_k\ell_{\mathrm{esc}}+8S_k$, which equals $\sum_tw_tL_t+8S_k$ for the weight vector $\mathbf{w}=(p_1,\ldots,p_{n-k},S_k)$ that treats the escape leaf as one further symbol of weight $S_k$. Package-Merge minimizes $\sum_tw_tL_t$ over all length assignments satisfying Kraft's inequality and the cap \citep{bounded_huffman}, so the candidate cost $c_k$ is the minimum over all codes whose escape set is $\{n-k+1,\ldots,n\}$. The loop evaluates every $k\in\{0,\ldots,n-1\}$, the omitted case $k=n$ being dominated: the all-eight code is Kraft-feasible and costs $8$, whereas escaping everything costs at least $9$. Keeping the smallest $c_k$ thus returns the optimum, with ties resolved to the smaller escape set by the strict comparison.

Canonical codes are assigned in code-length order. Direct symbols use their assigned code; escaped symbols use the escape prefix followed by their eight-bit table index, which identifies the centered exponent symbol rather than the full BF16 value. Encoding and decoding use precomputed lookup tables (1\,KiB and 4\,KiB) rather than tree traversal, and kernels apply the sampled tensor center when loading them, adapting to scale without repeating the cutoff search.

\begin{table}[t]
\centering
\caption{Implemented stream configurations. The final column is the fixed exponent payload per padded element, before overflow and metadata.}
\label{tab:geometry}
\small
\begin{tabular}{lrrrr}
\hline
Configuration & $B$ & $S$ & Bits/stream & Bits/element \\
\hline
Sparse (ReLU activations) & 65536 & 256 & 640 & 2.500 \\
Clean, non-Gaussian & 65536 & 256 & 800 & 3.125 \\
Clean, Gaussian & 65536 & 256 & 768 & 3.000 \\
Medium & 32768 & 128 & 608 & 4.750 \\
High & 32768 & 128 & 704 & 5.500 \\
\hline
\end{tabular}
\end{table}

The noise-level selector chooses the configuration in Table~\ref{tab:geometry}; it does not inject noise or change precision. If a symbol pair exceeds the stream budget, the overflow tail starts at that pair and all subsequent exponents are retained as raw bytes, with a start marker of 255 for streams without overflow.

\subsection{Overflow allocation and memory accounting}
\label{app:memory}

Kernels count overflowing streams and total tail length, compact their metadata with local prefix sums and atomic reservations, and gather raw exponent tails. Each affected stream records a 32-bit stream identifier, a 32-bit tail offset, and an eight-bit start step; raw tails store $e-127$ modulo 256 in one byte per exponent. Let $N_p$ be the padded element count, $R$ the total stream count, and $W$ the fixed number of 32-bit words per stream. If $K$ streams overflow and $T$ tail exponents are stored, the principal storage terms in bytes are
\begin{equation}
 M_{\mathrm{core}} = N_p + 4RW + T + 9K.
\end{equation}
Allocation additionally includes guard words, descriptors, alignment, cached codebooks, and transient workspaces. The nominal full-block rates follow from Equation~\ref{eq:allocated-ratio}; padding and overflow cause the departure from those rates and can produce expansion for small or poorly matched tensors.

For shared overflow storage, each allocation needs a sufficiently large contiguous free region, so capacity must account for concurrently live states and fragmentation. Per-tensor \texttt{memory\_size()} charges the tensor's aligned region; process-level reporting must instead count the entire shared arena once. The optional \texttt{allow\_raw} check returns dense storage when the fixed payload and padding alone exceed the input size; it does not guarantee against expansion after overflow. Appendix~\ref{app:overflow_buffer} describes allocation and storage lifetimes. Keeping an unneeded dense tensor alongside its compressed representation retains both allocations and reduces the storage benefit.

\subsection{Feed-forward autograd implementation}
\label{app:ffn-implementation}

Listing~\ref{lst:ffn-autograd} shows a custom PyTorch autograd function for a feed-forward layer. The compressed-tensor wrapper stores the input and post-ReLU activation using separate distribution descriptors; the buffer provides an optional shared overflow arena. Enabling compression changes how activations are saved and reconstructed for backward. With compression disabled, the same function retains dense activations. The final two \texttt{None} entries in the backward return correspond to the buffer and compression flag.

\begin{lstlisting}[
    float=p,
    language=Python,
    deletekeywords={buffer,bool},
    basicstyle=\ttfamily\footnotesize,
    columns=fullflexible,
    keepspaces=true,
    showstringspaces=false,
    breaklines=true,
    frame=single,
    caption={Feed-forward autograd function with optional lossless storage of saved activations.},
    label={lst:ffn-autograd},
    captionpos=t
]
class FFN(Function):
    """Dense baseline autograd FFN for comparison.

    For ``x[B, D]``, ``W1[H, D]``, and ``W2[D, H]`` computes
    ``z = relu(x @ W1.T)`` and ``output = z @ W2.T``.
    """

    @staticmethod
    def forward(ctx, x, W1, W2,
                buffer: TensorBuffer | None = None,
                compressed: bool = False):
        """Run the dense FFN forward pass and save tensors
        for backward.
        """
        z = x @ W1.T
        h = z.relu_()
        output = h @ W2.T

        if compressed:
            x = MyCompressed(x, buffer=buffer, dist=act_dist)
            h = MyCompressed(h, buffer=buffer, dist=act_relu_dist)

        ctx.save_for_backward(x, W1, W2, h)
        ctx.compressed = compressed
        return output

    @staticmethod
    def backward(ctx, grad_output):
        """Compute dense FFN gradients from ``grad_output[B, D]``."""
        needs_x = ctx.needs_input_grad[0]

        x, W1, W2, h = ctx.saved_tensors

        if ctx.compressed:
            x = x.decompress_free()
            h = h.decompress_free()

        grad_z = grad_output @ W2
        grad_W2 = grad_output.T @ h

        grad_preact = torch.ops.aten.threshold_backward.grad_input(
            grad_z, h, 0, grad_input=grad_z)

        del h, grad_z
        if not torch.compiler.is_compiling():
            ctx.maybe_clear_saved_tensors()

        if needs_x:
            grad_x = grad_preact @ W1
        else:
            grad_x = None
        grad_W1 = grad_preact.T @ x
        return grad_x, grad_W1, grad_W2, None, None
\end{lstlisting}

\subsection{Comparison with SplitZip} \label{app:split_zip}
SplitZip \citep{splitzip} is a GPU-native framework for lossless KV-cache compression. Like LCT, it compresses BF16 exponents using a codebook prepared before use. SplitZip assigns 4-bit codes to the 16 most common exponent values. Less common exponents are stored verbatim with a 16-bit integer marking their position. Its average storage per value, excluding other metadata and padding, is therefore
\begin{equation}
    b_{\mathrm{SplitZip}} = 8 + 4 + 24\epsilon = 12 + 24\epsilon \quad \text{bits/value},
\end{equation}
where the terms account for the sign--fraction byte, the 4-bit code, and overflow with probability $\epsilon$. With no overflow, the compression ratio is $16/12=4/3\approx1.33$. LCT achieves ratios above $1.4$ by assigning variable-length codes and optimizing which symbols use escapes. SplitZip stores a 16-bit position for each overflow value; LCT instead uses an escape prefix of at most eight bits followed by an eight-bit literal.

LCT also adapts its codebook through runtime exponent centering. SplitZip does not perform this adjustment, so changes in tensor scale can increase overflow under a fixed codebook. Adaptation is particularly useful during training, when tensor magnitudes evolve. In the evaluated SplitZip implementation, encoding synchronizes with the host to determine the overflow-buffer size, adding to encoding latency.

\section{Kernels}
\label{app:kernels}

The Triton kernels implement the codec described in Section~\ref{sec:standard-codec}. This section details their memory layout, bit packing, and allocation optimizations.

\subsection{Stream parallelism and table preparation}

Each lane retains its stream's bit position and partial words in registers. Payload words are stored with the stream index varying fastest, placing neighboring lanes' accesses to the same word position next to each other in memory.

Center estimation takes one deterministically jittered sample per input region. Integer hashing generates the offsets without a random-number buffer and reduces repeated sampling of the same matrix columns. Shifting the lookup tables once per tensor lets the encoder index raw exponent bytes and the decoder recover uncentered exponents directly, removing the centering transformation from the per-symbol loop.

\subsection{Encoding and decoding}
\label{app:bits}

\paragraph{Bit representation and invertibility.}

Let $u_i$ be the unsigned 16-bit pattern of element $i$. Its exponent byte $e_i$ and side byte $r_i$ are
\begin{align}
 e_i &= (u_i \mathbin{\gg} 7) \mathbin{\&} 255, \\
 r_i &= (u_i \mathbin{\&} 127)
       \mathbin{|} ((u_i \mathbin{\gg} 8) \mathbin{\&} 128).
\end{align}
The exponent is coded and $r_i$ is stored verbatim. Reconstruction,
\begin{equation}
 u_i = (e_i \mathbin{\ll} 7)
       \mathbin{|} (r_i \mathbin{\&} 127)
       \mathbin{|} ((r_i \mathbin{\&} 128) \mathbin{\ll} 8),
\end{equation}
recovers the original 16-bit pattern, including signed zeros, subnormals, infinities, and NaN payloads, without floating-point conversion.

\paragraph{Packing and lookup.}
The encoder processes two symbols per iteration. Each table entry stores the code bits and length; since a complete code is at most 16 bits, a pair occupies at most 32 bits and crosses at most one output-word boundary. Partial words remain in registers until written.

The decoder uses a 64-bit window spanning two payload words to decode a pair of symbols. Direct codes return an exponent and consumed length; escapes extract an eight-bit literal from the same window. In the full-block path, the next payload word is loaded only when a lane crosses a word boundary. Decode-table loads retain cache entries, side bytes use global-cache loads, and dense outputs use streaming stores. Loop unrolling and autotuned warp counts, pipeline depths, and register limits control instruction scheduling and register pressure.

\subsection{Overflow compaction and allocation}
\label{app:buffer}

Compaction uses local prefix sums to assign positions within each tile and atomic additions to reserve ranges across tiles. Reservations occur per tile rather than per symbol. Separate gather and scatter kernels access only the raw tails, keeping variable-length overflow handling outside the main packing and decoding loops.

\subsection{Overflow Buffer}
\label{app:overflow_buffer}

The allocator tracks occupied address ranges, uses first-fit search to reserve aligned regions, and writes device descriptors for subsequent kernels. Freeing a region marks it available; coalescing adjacent free regions reduces fragmentation. During an update, old and replacement allocations may coexist. On the same CUDA stream, the old region can be released after its final consumer is launched, preserving execution order.

\subsection{Fused pointwise operations}
\label{app:operations}

Table~\ref{tab:training-operators} summarizes the supported operations. Operator and output-format choices are resolved at compile time, eliminating branches from the inner loop. For a compressed and a dense operand, a correction kernel handles raw overflow tails.

\begin{table}[t]
\centering
\setlength{\tabcolsep}{4pt}
\renewcommand{\arraystretch}{1.15}
\caption{LCT operators and training integration points. $C$ and $D$ denote compressed and dense operands; $\alpha$ is a device scalar. Fused pointwise operators support dense or compressed output. Matrix products use dense GEMM after reconstruction.}
\label{tab:training-operators}
\small
\begin{tabular}{@{}>{\raggedright\arraybackslash}p{0.28\linewidth}>{\raggedright\arraybackslash}p{0.18\linewidth}>{\raggedright\arraybackslash}p{0.49\linewidth}@{}}
\hline
Operator & Computation & Training role \\
\hline
\texttt{compress} & $D\to C$ & Retain BF16 weights or state between uses; store activations within an explicitly managed backward integration. \\
\texttt{decompress} & $C\to D$ & Restore state or saved tensors for a dense consumer. \\
\texttt{compA\_add\_B} & $C+D$ & Accumulate microbatch gradients or apply a prepared parameter increment. \\
\texttt{compA\_mul\_B} & $C\odot D$ & Apply an existing elementwise mask, gate, or scale. \\
\texttt{a\_compA\_add\_B} & $\alpha C+D$ & Update momentum or moment state; combine weight decay with a prepared dense step. \\
\texttt{a\_compA\_add\_compB} & $\alpha C_1+C_2$ & Apply a compressed update to compressed state; $W-\eta U$ uses $U$ first and $W$ second. \\
\texttt{A\_compBT} & $DC^\top$ & Linear-layer forward product $XW^\top$ with compressed weights $W$. \\
\texttt{A\_compB} & $DC$ & Input-gradient product $G_YW$ with upstream gradient $G_Y$. \\
\hline
\end{tabular}
\end{table}

For two compressed operands, preparation maps stream identifiers to overflow offsets so each lane can select coded or raw exponents without searching compact metadata. Both operands must share stream geometry, escape length, and overflow-storage mode. For compressed output, sign--fraction bytes are written directly to their final locations, leaving only exponent bytes for subsequent centering and encoding.

\section{Derivations for Compression Analysis}
\label{app:theory}

\subsection{Cross entropy of the implemented code}

The implemented code lengths differ from ideal model lengths. Let $\ell_\gQ(t)$ include escape prefixes and literal bits, and define
\begin{equation}
 Z=\sum_t2^{-\ell_\gQ(t)}\leq1,\qquad
 \widetilde q_t=\frac{2^{-\ell_\gQ(t)}}{Z},
\end{equation}
where the inequality follows from the prefix-code property. Direct substitution yields the exact identity
\begin{equation}
 \bar\ell=H(\gP)+D_{\mathrm{KL}}(\gP\Vert\widetilde\gQ)-\log_2 Z.
\end{equation}
This separates source entropy, mismatch to the code-induced distribution, and unused code space, and it applies directly to LCT's length-limited escape construction. Every symbol retains a finite representation, including those assigned negligible probability by the analytic model.

\subsection{Exponent scaling and boundary effects}
\label{app:scale-transform}

\paragraph{Proof of Proposition~\ref{prop:exponent-scaling}: shift-corrected exponents.}
Let $X$ take finite, nonzero BF16 values and $a\ne0$ be a finite constant. Define the ideal exponents $E=\lfloor\log_2|X|\rfloor$, $E_a=\lfloor\log_2|aX|\rfloor$, and the common shift $k=\lfloor\log_2|a|\rfloor$. Excluding rounding and finite-range effects, the claim is
\begin{equation}
 (E_a-k)-E\in\{0,1\}.
 \label{eq:scale-shift}
\end{equation}
Write $\log_2|X|=E+U$ and $\log_2|a|=k+\delta$, where $0\leq U,\delta<1$. Then
\begin{equation}
 \lfloor\log_2|aX|\rfloor
 =E+k+\lfloor U+\delta\rfloor
 =E+k+J,\qquad J=\mathbf{1}\{U\geq1-\delta\}.
\end{equation}
The indicator is zero when $\delta=0$, including values exactly on a power-of-two boundary. More generally, let $p_j=\Pr(E=j)$ and $q_j=\Pr(E=j,J=1)$. The exact redistribution is
\begin{equation}
 \Pr(E_a=m)=p_{m-k}-q_{m-k}+q_{m-k-1}.
 \label{eq:exponent-bin-transfer}
\end{equation}
No independence or uniformity of the fractional log magnitude $U$ is assumed: the mass crossing a boundary can depend on the original bin. Given $E$, there are at most two possible values of $E_a$; given $E_a$, there are at most two possible values of $E$. Hence both conditional entropies are at most one bit, and the chain rule gives
\begin{equation}
 H(E_a)-H(E)=H(E_a\mid E)-H(E\mid E_a),
 \qquad |H(E_a)-H(E)|\leq1.
\end{equation}
For $\delta=0$, integer translation is a bijection, proving entropy invariance.

\paragraph{Proof of the BF16 boundary rule.}
For a finite real $y$, let $R(y)$ be the raw exponent byte after a single BF16 rounding to nearest, ties to even, with gradual underflow. The exact rule, including rounding, underflow, and overflow boundaries, is
\begin{equation}
 R(y)=\sum_{r=1}^{255}\mathbf{1}\{|y|\geq\tau_r\},\qquad
 \tau_r=\begin{cases}
 2^{-126}(1-2^{-8}), & r=1,\\
 2^{r-127}(1-2^{-9}), & 2\leq r\leq255.
 \end{cases}
 \label{eq:bf16-exponent-boundaries}
\end{equation}
Scaling replaces each threshold by $\tau_r/|a|$.

BF16 has seven fraction bits, minimum normal magnitude $2^{-126}$, and subnormal spacing $2^{-133}$ \citep{bf16,nvidiaFloatingPoint}. For an interior normal power $2^j$, $-125\leq j\leq127$, the immediately preceding value is $2^j-2^{j-8}$. Their rounding midpoint is therefore $2^j-2^{j-9}$. The midpoint below the minimum normal value is instead $2^{-126}-2^{-134}$ because the subnormal spacing is constant. At each midpoint the power of two has the even significand, so ties round into the higher exponent bin. At the upper end, the maximum finite value is $2^{128}-2^{120}$; rounding overflows at $2^{128}-2^{119}$, including the tie. These are exactly the thresholds in Equation~\ref{eq:bf16-exponent-boundaries}. Counting the thresholds crossed gives the raw exponent byte, proving the rule.

The interval below the first threshold contains both subnormals and zero; magnitudes at most $2^{-134}$ round to zero, including the tie. For a continuous magnitude density $g$ and $a\ne0$, the normal-byte probabilities are consequently
\begin{equation}
 \Pr(R(aX)=r)=\int_{\tau_r/|a|}^{\tau_{r+1}/|a|}g(x)\,dx,
 \qquad 1\leq r\leq254.
 \label{eq:rounded-exponent-probabilities}
\end{equation}
Byte zero collects the interval $[0,\tau_1/|a|)$, and byte 255 collects $[\tau_{255}/|a|,\infty)$. For discrete inputs or point masses, these are half-open probability intervals with the same endpoint convention, rather than density integrals. An exact zero remains zero under finite nonzero scaling. A zero multiplier sends all finite inputs to exponent byte zero; NaNs and infinite inputs are outside the propositions. Flush-to-zero arithmetic or an intermediate rounding step requires its own boundary rule. LCT itself preserves the resulting BF16 bits; the rounding here belongs to the arithmetic producing them.

\paragraph{Proof of Proposition~\ref{prop:boundary-entropy}: boundary entropy correction.}
For the proposition's nonzero input, set $A=R(X)$ and $B=R(aX)$. Let $D$ be the event that $A\ne E+127$ or $B\ne E_a+127$, and write $\varepsilon=\Pr(D)$. This event captures departures from the ideal exponent labels due to rounding and finite range. On $D^c$, the same two-bin argument gives $|H(B\mid D=0)-H(A\mid D=0)|\leq1$. On $D$, each byte has only 256 possible values, so the corresponding difference is at most eight bits. Expanding entropy by the binary event $D$ yields
\begin{equation}
 |H(B)-H(A)|\leq
 \min\{8,\;1+7\varepsilon+h_2(\varepsilon)\},
 \label{eq:boundary-entropy-bound}
\end{equation}
where $h_2$ is binary entropy. Indeed, the weighted conditional differences contribute at most $(1-\varepsilon)+8\varepsilon$, and the remaining term $I(B;D)-I(A;D)$ has absolute value at most $H(D)=h_2(\varepsilon)$. Dropping the trivial eight-bit cap gives the bound in Proposition~\ref{prop:boundary-entropy}; its correction vanishes as $\varepsilon\to0$. If scaling is by a power of two, the conditional difference on $D^c$ is zero and the bound improves to $\min\{8,8\varepsilon+h_2(\varepsilon)\}$.

For actual normal BF16 inputs, power-of-two scaling changes only the exponent when the output stays normal and finite, and is exactly representable. Reusing the same centered symbols also requires the estimated center to shift by that integer. Sampling, the half-integer rounding convention, center clamping, and reserved-zero remapping can affect this correspondence. Thus neither the proposition nor its entropy bound implies an unchanged code length for every implementation configuration.

\subsection{Exponent distributions from continuous densities} \label{app:exponent-distributions}
This section studies the distribution of exponents for data drawn from a particular data distribution.
\paragraph{Continuous change of parameterisation.}
We first derive the formula for the distribution of exponents. Let $X>0$ have density $f_X$ and first use natural logarithms, $e=\ln X$. Since $X=\exp(e)$ and $dx/de=\exp(e)$, change of variables gives
\begin{equation}
 f_e(t)=\exp(t)f_X(\exp(t)).
 \label{eq:log-density-natural}
\end{equation}
For signed tensor values, take the magnitude $M=|X|$. For a continuous signed density, $g(x)=f_X(x)+f_X(-x)$ for $x>0$ is its magnitude density. Using base two, $V=\log_2 M$, gives
\begin{equation}
 f_V(v)=(\ln2)\,2^v g(2^v),\qquad
 \Pr(\lfloor V\rfloor=j)
 =\int_j^{j+1}f_V(v)\,dv
 =\int_{2^j}^{2^{j+1}}g(x)\,dx.
 \label{eq:log-density-binary}
\end{equation}
The Jacobian factor $2^v\ln2$ accounts for the increasing width of value intervals. Equation~\ref{eq:log-density-binary} gives continuous exponent-bin masses while Equation~\ref{eq:rounded-exponent-probabilities} gives the finite-format correction. The entropy of the discrete masses determines the entropy coding cost. Scaling translates the continuous density exactly, $f_{V_a}(v)=f_V(v-\log_2|a|)$.

\paragraph{Why observed exponent distributions are concentrated and asymmetric.}
Next, we discuss the shapes of exponent distributions.
\begin{lemma}[Left-tail decay of log magnitudes]
\label{lem:log-left-tail}
Let $M>0$ have density $g$ with a power-law upper bound near zero: $g(x)=O(x^{\alpha-1})$ as $x\downarrow0$ for some $\alpha>0$. For $V=\log_2M$,
\begin{equation}
 f_V(v)=O(2^{\alpha v})\quad(v\to-\infty),\qquad
 \Pr(\lfloor V\rfloor=j)=O(2^{\alpha j})\quad(j\to-\infty).
 \label{eq:log-left-tail}
\end{equation}
Both the density and the exponent-bin probabilities therefore decay at least exponentially to the left.
\end{lemma}

\noindent\emph{Proof.}
There are constants $C,x_0>0$ such that $g(x)\leq Cx^{\alpha-1}$ for $0<x<x_0$. Equation~\ref{eq:log-density-binary} then gives $f_V(v)\leq C\ln2\,2^{\alpha v}$ whenever $2^v<x_0$. For $2^{j+1}<x_0$, integrating the same bound gives
\begin{equation}
 \Pr(\lfloor V\rfloor=j)
 \leq C\int_{2^j}^{2^{j+1}}x^{\alpha-1}\,dx
 =\frac{C(2^\alpha-1)}{\alpha}\,2^{\alpha j}.
\end{equation}
These inequalities prove both bounds.

Any density bounded near zero satisfies the lemma with $\alpha=1$, giving $f_V(v)=O(2^v)$. The lemma also covers densities that diverge at zero when they satisfy the stated bound for some $\alpha>0$. BF16 groups the smallest magnitudes into exponent byte zero as subnormal numbers or zeros.

\begin{lemma}[Right-tail decay of log magnitudes]
\label{lem:log-right-tail}
Let $M>0$ have density $g$ satisfying $g(x)=O(x^{-1-\beta})$ as $x\to\infty$ for some $\beta>0$. For $V=\log_2M$,
\begin{equation}
 f_V(v)=O(2^{-\beta v})\quad(v\to\infty),\qquad
 \Pr(\lfloor V\rfloor=j)=O(2^{-\beta j})\quad(j\to\infty).
 \label{eq:log-right-tail}
\end{equation}
Both the density and the exponent-bin probabilities therefore decay at least exponentially to the right.
\end{lemma}

\noindent\emph{Proof.}
There are constants $C,x_0>0$ such that $g(x)\leq Cx^{-1-\beta}$ for $x>x_0$. Equation~\ref{eq:log-density-binary} gives $f_V(v)\leq C\ln2\,2^{-\beta v}$ whenever $2^v>x_0$. For $2^j>x_0$,
\begin{equation}
 \Pr(\lfloor V\rfloor=j)
 \leq C\int_{2^j}^{2^{j+1}}x^{-1-\beta}\,dx
 =\frac{C(1-2^{-\beta})}{\beta}\,2^{-\beta j}.
\end{equation}
These inequalities prove both bounds.

A magnitude density satisfying the stated power-law bound therefore has an exponentially decaying right tail in its exponent distribution. In practice, BF16 values that overflow share exponent byte 255.

An example of these effects can be seen in the Gaussian $X\sim\mathcal{N}(0,\sigma^2)$,
\begin{equation}
 \label{eq:gaussian-laplace-log-densities}
 f_V(v)=\frac{\sqrt{2/\pi}\ln2}{\sigma}\,2^v
       \exp\!\left(-\frac{2^{2v}}{2\sigma^2}\right).
\end{equation}
The peak is at $v=\log_2\sigma$ with an exponentially decaying left tail and super-exponentially decaying right tail. Furthermore, changing $\sigma$ shifts the log-density without changing its shape.

\subsection{Calculation of theoretical coding and storage costs}
\label{app:entropy-matrix}

We use the Gaussian, Laplace, power-law, and Gamma models compared in Table~\ref{tab:exponent-cross-entropy}, with population exponent centering and no added zero mass. Exponent probabilities are analytically calculated.

The left panel of Table~\ref{tab:exponent-cross-entropy} gives the ideal cross entropy $H(\gP,\gQ)=-\sum_t p_t\log_2q_t$. The right panel gives predicted allocated exponent storage under the \texttt{CLEAN} configuration. The code lengths $\ell_\gQ(t)$ include escape literals, with mean $\bar\ell=\sum_t p_t\ell_\gQ(t)$. Allocated storage also depends on the fixed stream budget and overflow: raw tails use one byte per exponent, with nine metadata bytes per overflowing stream, as in Equation~\ref{eq:allocated-ratio}. The calculation assumes independent symbols and full blocks, and excludes padding and allocation overhead. Neither panel includes the unchanged sign--fraction byte.

\paragraph{Effect of tail mismatch.}
The source and coding models can have similar central mass but very different tails. The power-law source places substantially more probability on large magnitudes than the Gaussian model, whose tail decays exponentially in the squared magnitude. Consequently, large exponent symbols can have very small model probabilities $q_t$, making their ideal costs $-\log_2q_t$ extremely large. These terms dominate the power-law-to-Gaussian cross entropy, producing the large entry in Table~\ref{tab:exponent-cross-entropy}; it is a mismatch penalty, not high source entropy. LCT limits each exponent code to at most 16 bits, including the escape literal, and stores overflowing tails as raw bytes. Its storage cost is therefore less sensitive to this mismatch. This does not violate the entropy bound: the comparison is with a mismatched ideal coder, not an optimal coder for the source.

\subsection{Entropy of exponent and sign--fraction byte}
\label{app:full-value-entropy}
For typical floating-point data, almost all the compressibility is in the exponent. Let BF16 samples from a distribution be represented as $U=(E,R)$, where $E$ is the exponent byte and $R$ contains the sign and seven fraction bits. Exponent-only coding stores $R$ verbatim and compresses $E$ independently of $R$, at a cost of $H(E)+8$ bits per value, while ideal full-value coding costs $H(E,R)$. The additional saving is
\begin{equation}
 \Delta b=H(E)+8-H(E,R)=8-H(R\mid E).
\end{equation}
For a symmetric continuous distribution rounded to BF16, the sign bit is uniform and independent of the magnitude. When the fraction bits also have high conditional entropy, $H(R\mid E)$ is close to eight bits and the additional saving is small.
For example, for the Gaussian distribution $X\sim\mathcal{N}(0,1)$ rounded to BF16, the conditional entropy is $H(R\mid E)\approx7.9195$ bits. Coding the full value therefore saves only $\Delta b\approx0.0805$ bits per value beyond exponent-only coding.

The sign--fraction byte can have lower entropy in some cases. One example is nonnegative distributions whose sign bit is always zero. This is the case for certain activation functions such as \texttt{ReLU} and \texttt{sigmoid}. Another case is if the mean of the distribution is large and all values cluster around a single exponent, e.g. $\mathcal{N}(12, 1)$ where values are mostly distributed between $2^{3}$ and $2^4$ so there is no pooling of mantissa values (and exponents have very low entropy). LCT does not take advantage of a low-entropy sign--fraction byte if this is the case.

\section{Nemotron Weight and Activation Distributions}
\label{app:measurements}

\subsection{Weight exponent distributions}
\label{app:weight-exponents}

Figure~\ref{fig:weight-exponents} shows BF16 exponent distributions for six Nemotron-H-8B parameter groups, each containing more than 500M values. Histograms include all weights and are pooled across layers. Most probability mass lies near exponents $-7$ and $-8$, supporting exponent-only compression.

\begin{figure}[t]
 \centering
 \includegraphics[width=\linewidth]{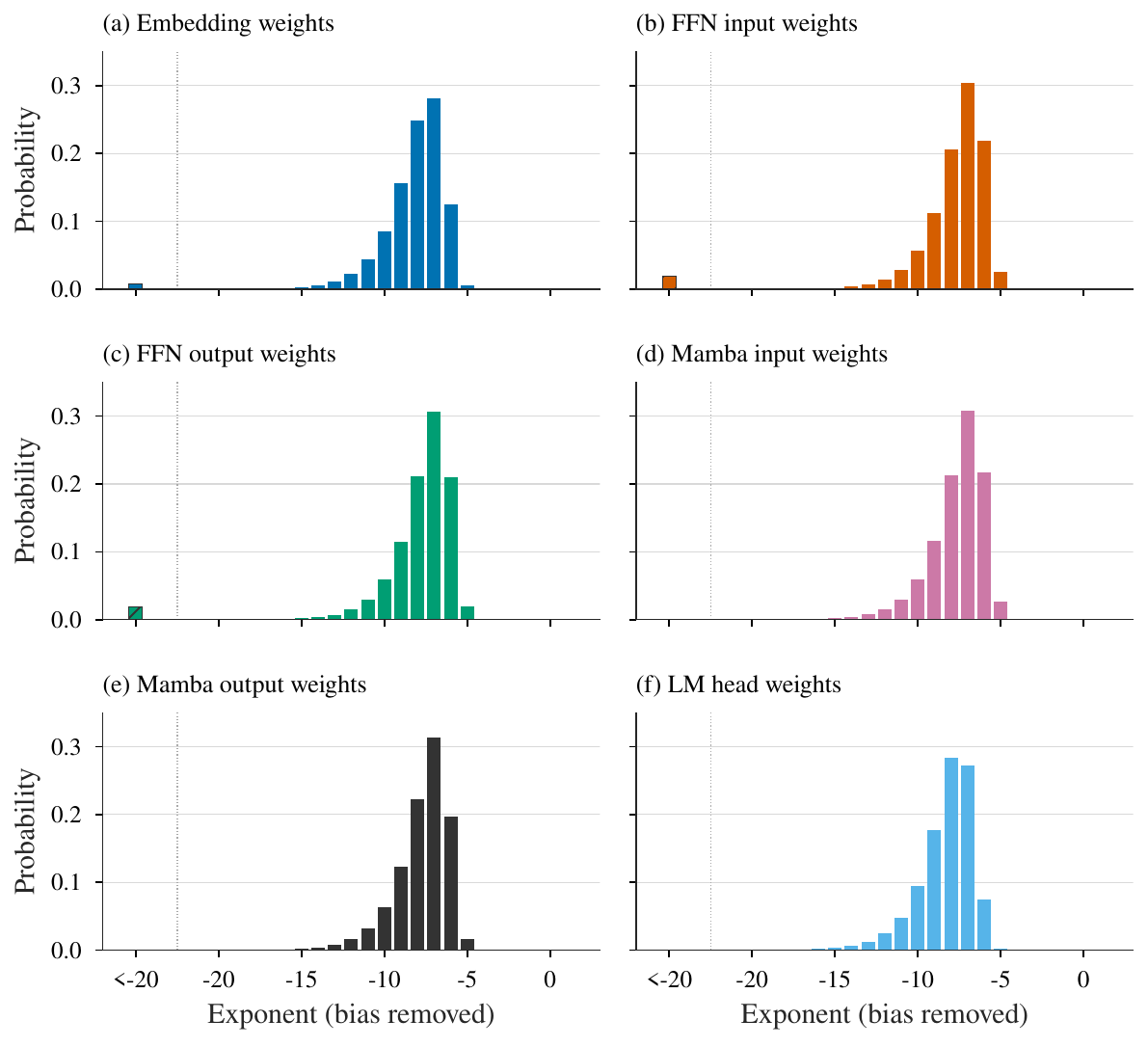}
 \caption{Nemotron-H-8B weight exponent probabilities, with the BF16 bias removed. The hatched bar aggregates exponents below $-20$. Probabilities are normalized over all weights.}
 \label{fig:weight-exponents}
\end{figure}

\subsection{Activation exponent distributions}
\label{app:activation-exponents}

We collect activations in BF16 evaluation mode on four 1024-token sequences from FineWeb validation data, tokenized for Nemotron. Figure~\ref{fig:activation-exponents} pools counts across layers for the five groups with more than 500M observations. Exponents remain concentrated, although their spread varies by activation type. FFN hidden activations are particularly compressible, with 93.23\% exact zeros.

\begin{figure}[t]
 \centering
 \includegraphics[width=\linewidth]{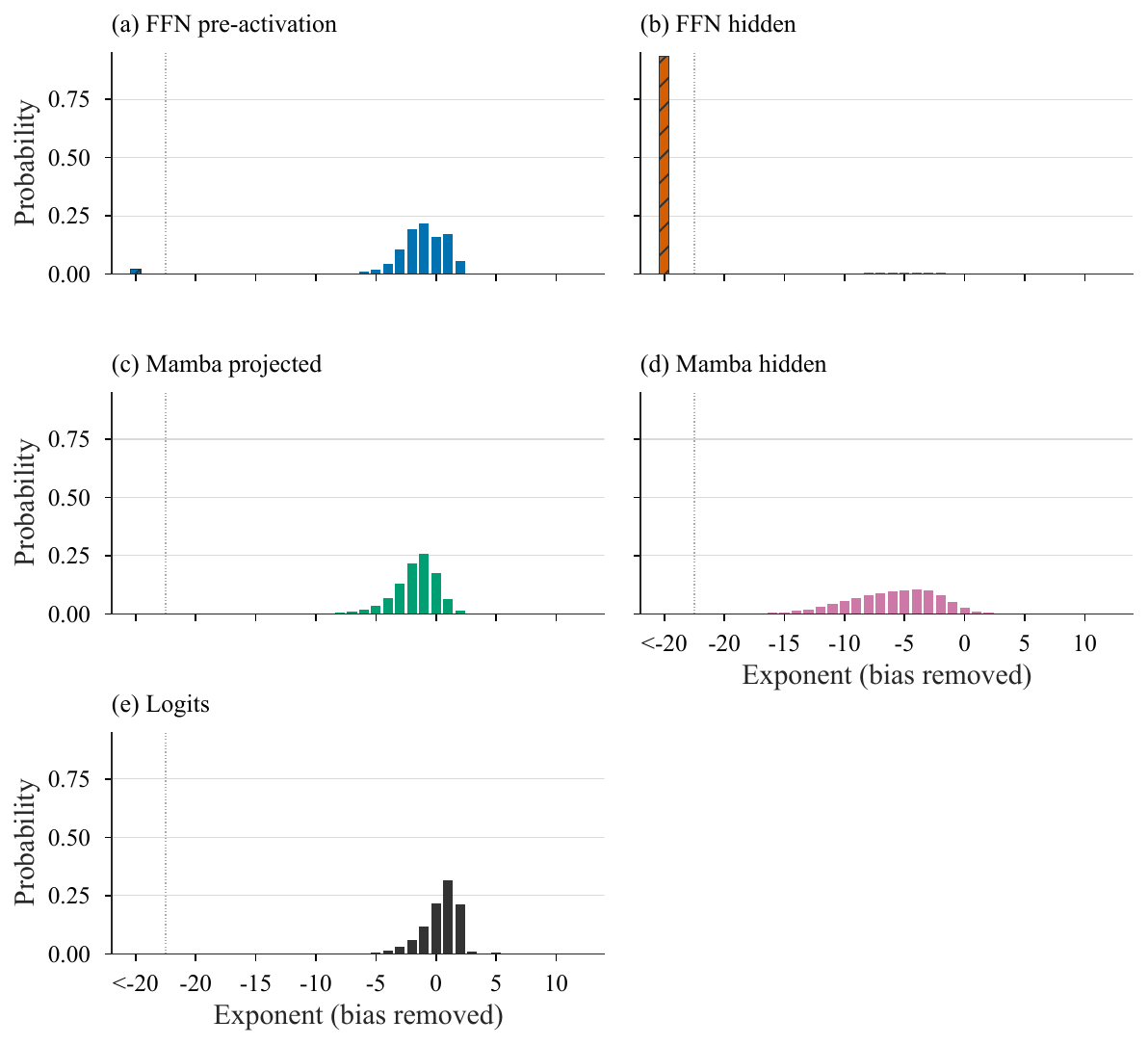}
 \caption{Nemotron-H-8B activation exponent probabilities, with the BF16 bias removed. The hatched bar combines exponents below $-20$, exact zeros, and subnormals; the FFN hidden bar is dominated by exact zeros. Probabilities are normalized over all observations.}
 \label{fig:activation-exponents}
\end{figure}

\end{document}